\documentclass[10pt,twocolumn,letterpaper]{article}

\usepackage[pagenumbers]{cvpr} 

\usepackage{graphicx}
\usepackage{amsmath}
\usepackage{amssymb}
\usepackage{booktabs}

\usepackage{multirow}
\usepackage{color,colortbl}
\definecolor{blush}{rgb}{0.87, 0.36, 0.51}
\definecolor{bblue}{rgb}{0.36, 0.51, 0.87}
\definecolor{ggray}{rgb}{0.88, 0.87, 0.87}
\newcolumntype{a}{>{\columncolor{ggray}}c}
\usepackage{flushend}
\usepackage[pagebackref,breaklinks,colorlinks]{hyperref}

\usepackage[capitalize]{cleveref}
\crefname{section}{Sec.}{Secs.}
\Crefname{section}{Section}{Sections}
\Crefname{table}{Table}{Tables}
\crefname{table}{Tab.}{Tabs.}

\usepackage{bbm}
\newtheorem{theorem}{Theorem}
\newtheorem{lemma}[theorem]{Lemma}
\newenvironment{proof}{\paragraph{\emph{Proof}:}}{\hfill$\square$}
\usepackage[dvipsnames, svgnames, x11names]{xcolor}
\def\cvprPaperID{***}
\def\confName{CVPR}
\def\confYear{2022}

\begin{document}
	
	\title{UMAD: Universal Model Adaptation under Domain and Category Shift}
	
	\author{Jian Liang
				 $^{1,2}$
				\qquad 
				Dapeng Hu\thanks{The first two authors share co-first authorship.} $^{3}$
				\qquad
				Jiashi Feng $^{4}$
				\qquad
				Ran He $^{1,2}$\\
				$^1$ Institute of Automation, Chinese Academy of Sciences (CASIA)
				\quad
				$^2$ University of Chinese \\ Academy of Sciences (UCAS)
				\quad
				$^3$ National University of Singapore (NUS)
				\quad
				$^4$ Sea AI Lab (SAIL) \\
				{\tt\small \color{black}{liangjian92@gmail.com}}\qquad
				{\tt\small dapeng.hu@u.nus.edu}\qquad
				{\tt\small jshfeng@gmail.com}\qquad
				{\tt\small rhe@nlpr.ia.ac.cn}
	}
	\maketitle
	
	\begin{abstract}
		Learning to reject unknown samples (not present in the source classes) in the target domain is fairly important for unsupervised domain adaptation (UDA).
		There exist two typical UDA scenarios, \ie, open-set, and open-partial-set, and the latter assumes that not all source classes appear in the target domain.
		However, most prior methods are designed for one UDA scenario and always perform badly on the other UDA scenario.
		Moreover, they also require the labeled source data during adaptation, limiting their usability in data privacy-sensitive applications. 
		To address these issues, this paper proposes a Universal Model ADaptation (UMAD) framework which handles both UDA scenarios without access to the source data nor prior knowledge about the category shift between domains. 
		Specifically, we aim to learn a source model with an elegantly designed two-head classifier and provide it to the target domain.
		During adaptation, we develop an informative consistency score to help distinguish unknown samples from known samples.
		To achieve bilateral adaptation in the target domain, we further maximize localized mutual information to align known samples with the source classifier and employ an entropic loss to push unknown samples far away from the source classification boundary, respectively. 
		Experiments on open-set and open-partial-set UDA scenarios demonstrate that UMAD, as a unified approach without access to source data, exhibits comparable, if not superior, performance to state-of-the-art data-dependent methods.
		
	\end{abstract}
	
	\begin{figure}[t]
		\begin{center}
			\includegraphics[width=0.48\textwidth, trim=0 0 0 0,clip]{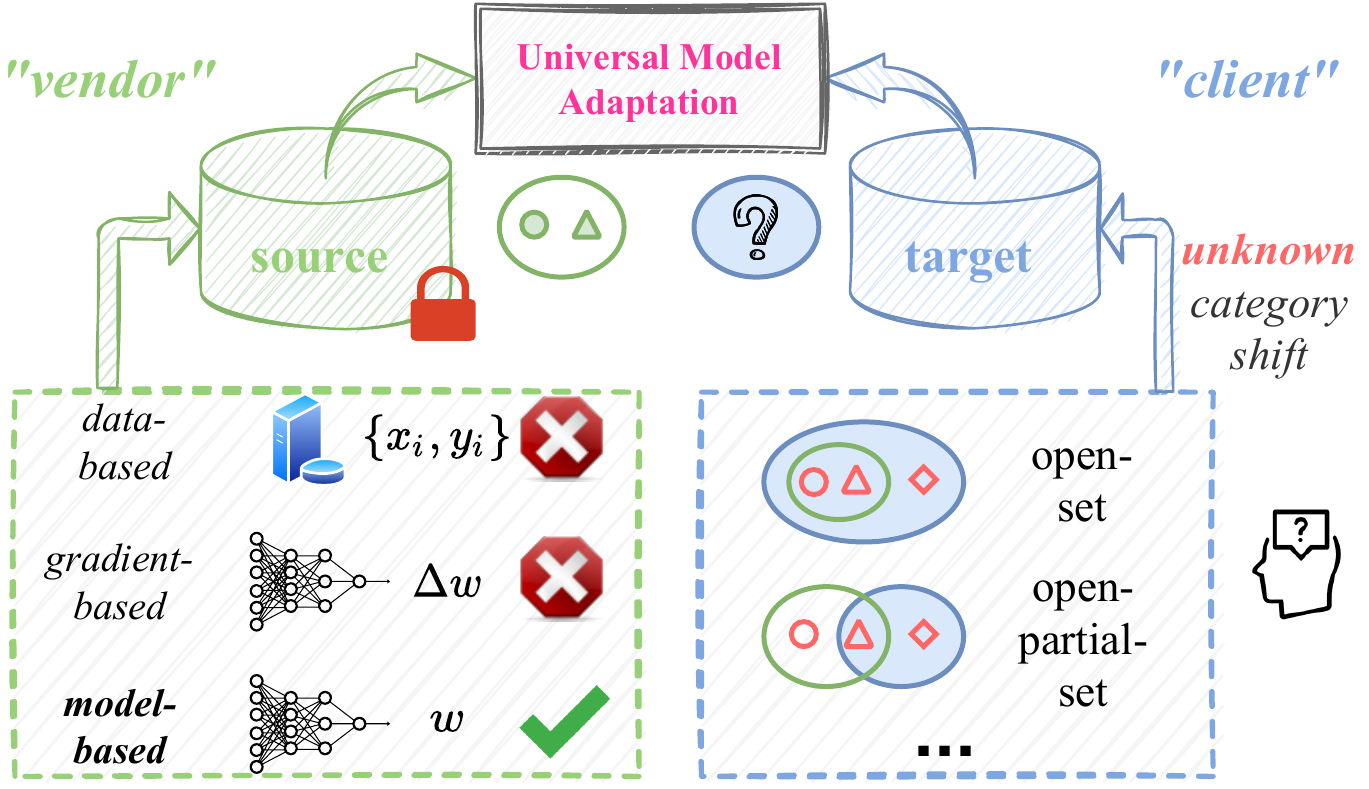}
			\caption{To pursue a broad range of applicability, we consider a challenging problem where the trained model instead of the data is provided from the source domain and meanwhile the detailed category shift (\ie, open-set or open-partial-set) is unknown.}
			\label{fig:framework}
		\end{center}
	\vspace{-20pt}
	\end{figure}
	
	\section{Introduction}
	Benefiting from the increasing number of labeled data, 
	deep learning has achieved great success, particularly in supervised learning.
	However, annotating data in a new situation is always time-cost and expensive \eg semantic segmentation in autonomous driving \cite{feng2020deep}.
	Thus, researchers resort to an alternative solution by fully exploiting knowledge from data or models available in similar domains (referred to transfer learning (TL) \cite{pan2009survey}).
	As a particular case of TL, single-source unsupervised domain adaptation (UDA) \cite{ben2010theory} attracts increasing attention over the last decade, which has been successfully applied in many real-world applications, \eg, image classification \cite{ganin2015unsupervised,long2015learning}, semantic segmentation \cite{tsai2018learning,hoffman2018cycada}, and object detection \cite{chen2018domain,kim2019diversify}.
	
	Typically, a UDA problem involves two related but not identical domains: source and target domains.
	The goal of UDA is to recognize the unlabeled samples in the target domain by leveraging knowledge from labeled data in the source domain.
	For a better illustration, we denote $\mathcal{C}_s$ and $\mathcal{C}_t$ as the label set of the source domain and the target domain, respectively.
	This well-studied identical label space assumption \cite{ganin2015unsupervised} ($\mathcal{C}_s = \mathcal{C}_t$, known as closed-set) is recently extended to various asymmetric settings, \eg, partial-set \cite{cao2018partial} ($\mathcal{C}_s \supset \mathcal{C}_t$), open-set \cite{saito2018open} ($\mathcal{C}_s \subset \mathcal{C}_t$), and open-partial-set \cite{you2019universal} ($\mathcal{C}_s \cap \mathcal{C}_t \not= \O, \mathcal{C}_s\not \supseteq \mathcal{C}_t, \mathcal{C}_s\not \subseteq \mathcal{C}_t$).
	Despite showing favorable results on one single setting with specific category overlap, previous UDA methods can not guarantee generalization to other settings.
	For example, a open-partial-set method \cite{you2019universal} even under-performs the baseline source-only method for partial-set UDA \cite{saito2020universal}.
	A pioneering work based on self-supervision is proposed in \cite{saito2020universal} to handle arbitrary settings across domains, however, it cannot well handle unknown samples in the target domain.
	Besides, with the increasing attention on data privacy, recent studies like \cite{liang2020we,li2020model,kundu2020universal} attempt to protect personal private data by only offering models trained in the source domain.
	For instance, the car company (\emph{vendor}) in self-driving will only deploy the system but not share the training data.
	Assuming the presence of unknown samples in the target domain, this paper focuses on a new but realistic UDA problem illustrated in Fig.~\ref{fig:framework} where the trained model instead of the source data is provided and the detailed category shift (\ie, open-set or open-partial-set) is unknown meanwhile.
	
	Generally, there are three challenges in such a source data-free universal adaptation problem, \ie,
	i) how to train a source model and adapt it to the unlabeled target domain; 
	ii) how to guarantee the adaptation performance when some source classes may be absent in the target domain;
	and iii) how to guarantee the adaptation performance when some unknown classes exist in the target domain.
	To tackle them together, we propose a unified framework called Universal Model ADaptation (UMAD).
	Inspired by prior model adaptation methods~\cite{liang2020we,liang2021source} which are mainly tailored to closed-set UDA, we also first train a self-defined source model and then provide it to the target domain for separate adaptation.
	Differently, we design a two-head classifier within the source model and introduce an informative consistency score to measure the uncertainty of target samples. 
	With this score, we then propose a novel rejection mechanism to distinguish unknown samples from known samples in the target domain, where the key threshold is automatically determined by the averaged score of interpolated target samples via Mixup \cite{zhang2018mixup}. 
	As for the bilateral model adaptation, we first reject unknown samples by pushing them far away from the source classification boundary. 
	To align known samples with the corresponding source classifier, we then propose a new localized mutual information objective with a generalized diversity term.
	
	\noindent \textbf{Contributions.} We study the problem of universal model adaptation, with unknown samples in the target domain and no access to the source data. We make the following contributions:
	1). we propose a simple yet effective model adaptation framework that aligns known samples against domain shift and rejects unknown samples under unknown category shift;
	2). we devise a new information consistency score and an automatic thresholding scheme to reject unknown samples.
	3). we propose a localized mutual information objective to implicitly align known samples with the source model under various category shifts.
	4). we validate the effectiveness of each component in UMAD via extensive experiments.
	Our source data-free UMAD even exhibits comparable performance to state-of-the-art data-dependent methods on open-set and open-partial-set UDA tasks.
	
	\section{Related Work}
	\textbf{Unsupervised domain adaptation.}
	Before the deep learning era, researchers utilize hand-crafted features and always resort to instance weighting \cite{zadrozny2004learning,sugiyama2007direct,huang2006correcting} or feature alignment \cite{gong2012geodesic,sun2016return,long2013transfer} to address covariate shift.
	In recent years, benefiting from representation learning, deep domain adaptation methods \cite{tzeng2014deep,ganin2015unsupervised,long2015learning,tzeng2017adversarial} have almost dominated this field.
	To mitigate the gap across domains, there are two prevailing feature-level paradigms, \ie, adversarial training \cite{ganin2015unsupervised} and discrepancy minimization \cite{long2015learning}.
	By contrast, other studies \cite{saito2018maximum,jin2020minimum,liang2020we} focus on the network outputs and expect the target data can be structurally close to the classification boundary.
	Some other works investigate the network components like batch normalization \cite{li2018adaptive} and dropout \cite{lee2019drop} or study the properties of learned features \cite{xu2019larger,chen2019transferability}.
	
	Most aforementioned UDA methods focus on the closed-set scenario where the source label set is the same as the target label set. 
	To be more practical, the open-set UDA scenario is first introduced in \cite{panareda2017open} where unknown samples exist in both domains.
	The following work \cite{saito2018open} further assumes the presence of unknown samples only in the target domain, and it allows extracting features that separate the unknown target from known target samples via adversarial training.
	Mixing the properties of open-set and partial-set \cite{cao2018partial} together, \cite{you2019universal} introduces a realistic scenario termed ``universal" UDA.
	It exploits both the domain similarity and the prediction uncertainty of each sample to develop a weighting mechanism for discovering label sets shared by both domains and thus promote the common-class adaptation. 
	In a narrow sense, \cite{you2019universal} only deals well with open-partial-set UDA but worse with other special cases like open-set and partial-set UDA scenarios.
	A recent work \cite{saito2020universal} proposes a truly universal UDA method that works well on all four scenarios, yet it performs badly for open-set and open-partial-set UDA.
	Generally, different from the well-studied closed-set scenario, these new scenarios are more challenging due to the asymmetric label set across domains, which are prone to cause the negative transfer.
	We focus on developing a unified method for adaptation to a domain with unknown samples, \ie, open-set and open-partial-set UDA.
	
	\begin{figure*}[!htb]
		\centering
		\small
		\setlength\tabcolsep{1mm}
		\renewcommand\arraystretch{0.35}
		\begin{tabular}{cc}
			\includegraphics[width=0.42\linewidth,height=1.7in,trim={0cm 0cm 0.0cm 0.0cm}, clip]{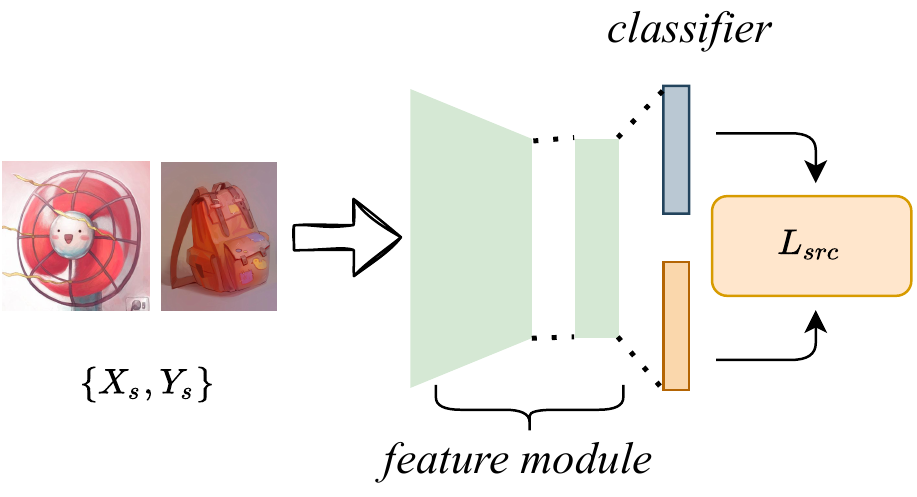} &
			\quad
			\includegraphics[width=0.58\linewidth,height=1.7in,trim={0cm 0cm 0.0cm 0.0cm}, clip]{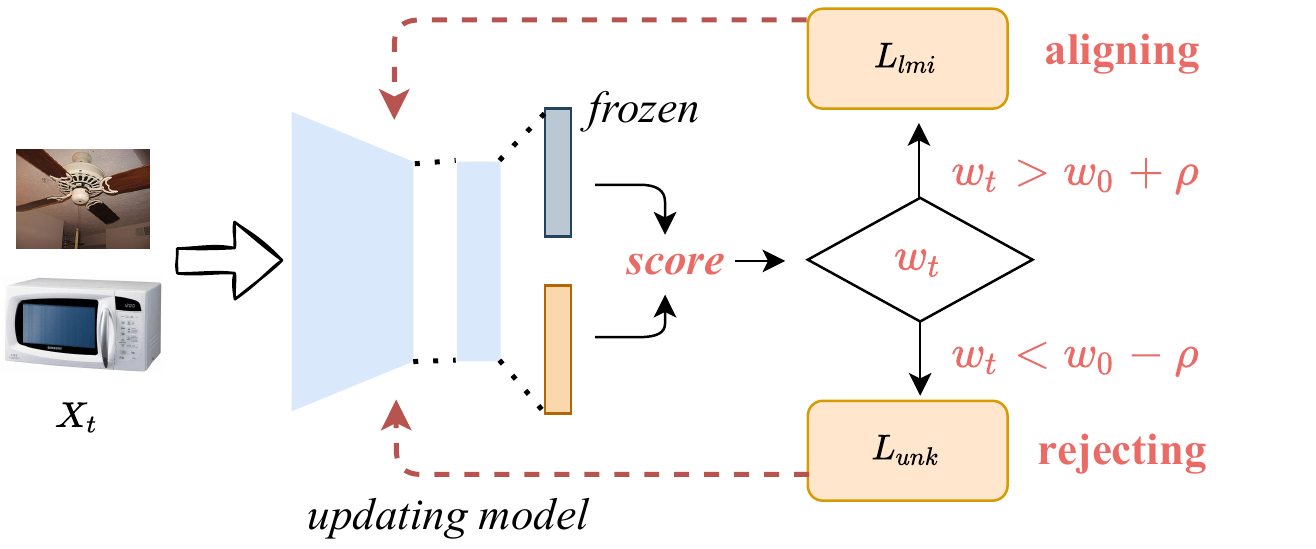} \\
			~\\
			(a) \textbf{model training} in the source domain  & (b) \textbf{model adaptation} in the target domain
		\end{tabular}
		\caption{Overview of Universal Model ADaptation (UMAD). UMAD trains a deep neural network with a two-head classifier and provides the trained model to the target domain. During adaptation, UMAD freezes the two-head classifier and fine-tunes the source feature module from two different directions, \ie, aligning known samples against domain shift and rejecting unknown samples against category shift.}
		\label{fig:UMAD}
	\end{figure*} 
	
	\textbf{Source data-free domain adaptation.}
	Besides the attention on the varying label set in the target domain, researchers also consider domain adaptation in the absence of source data due to data privacy concerns and copyright laws.
	A pioneering work \cite{chidlovskii2016domain} comes up with several solutions using the source classifier, and \cite{liang2019distant} requires the statistics information like means and co-variances during feature alignment. 
	However, they do not exploit representation learning but use fixed features instead.
	Inspired by the idea of model transfer \cite{mansour2008domain,kuzborskij2013stability,nelakurthi2018source}, \cite{liang2020we} proposes a new framework where only the source feature module is fine-tuned in the target domain to achieve domain alignment.
	Meanwhile, \cite{li2020model} proposes clustering regularization and weight regularization that encourages similarity to the source model, and \cite{kundu2020universal,kundu2020towards} mainly focus on building a stronger pre-trained source model by generating unknown samples.
	Following works address this problem via variational inference \cite{yeh2021sofa}, generative adversarial network \cite{kurmi2021domain}, and mutual information \cite{lao2021hypothesis}, respectively.
	In addition, \cite{peng2020federated} exploits federated learning and transfers knowledge from the decentralized nodes to a new node with a different data domain.
	However, as evidenced in \cite{zhu2019deep}, sharing the gradient information is still risky.
	Besides, these methods above are designed for only one specific UDA scenario, \eg, \cite{li2020model} for closed-set UDA, \cite{kundu2020towards} for open-set UDA, and \cite{kundu2020universal} for open-partial-set UDA.
	Even \cite{liang2020we} provides results for UDA under three different scenarios, it still needs to know the UDA scenario in advance and subsequently adjusts the algorithm.
	By contrast, our method is a unified framework for both open-set and open-partial-set UDA with the same parameters. 
	
	\textbf{Open-set learning.}
	Open-set learning tries to handle ``unknown" classes that are not contained in the training set, which is vital when deployed to real applications.
	\cite{bendale2016towards} proposes an approach that builds per-class probabilistic models of the input not belonging to the known classes and combines these in the final estimate of each class probability including the unknown class.
	\cite{hendrycks2016baseline} shows that the maximum of the softmax outputs, or confidence, can be used to detect out-of-distribution (OOD) inputs.
	\cite{liang2018enhancing} exploits temperature scaling and perturbs an input in the direction of maximally increasing the max-softmax.
	Besides the aleatoric uncertainty, \cite{gal2016dropout} employs Bayesian neural networks to consider epistemic uncertainty, which is computationally more efficient than an ensemble of networks \cite{lakshminarayanan2017simple}.
	\cite{lee2018simple} models the distribution of intermediate layer’s activation by a Gaussian distribution for each class,
	and the OOD score is given by Mahalanobis distance and logistic regression, assuming the availability of OOD samples.
	Inspired by \cite{saito2018maximum}, the discrepancy between two classifiers is utilized to reject unknown samples in \cite{yu2019unsupervised}.
	Our method differs from \cite{yu2019unsupervised} in that we propose a new information consistency score and an automatic thresholding scheme, which works under both domain shift and category shift.
	
	\section{Methodology}
	In a source data-free UDA task, there exist a source domain $\mathcal{D}_s=\{(x_i^s,y_i^s)\}_{i=1}^{N_s}$ consisting of $N_s$ labeled samples and a target domain $\mathcal{D}_t=\{(x_i^t)\}_{i=1}^{N_t}$ consisting of $N_t$ unlabeled samples, and $\mathcal{D}_s$ is not accessible during adaptation to the target domain. 
	$\mathcal{C}_s \ni y_i^s$ and $\mathcal{C}_t \ni y_i^t$ denote the label sets of the source domain and the target domain, respectively. 
	This paper mainly focuses on $\bar{\mathcal{C}}_t=\mathcal{C}_t \setminus \mathcal{C}_s \not= \O$, which contains two open-set scenarios as special cases in Fig.~\ref{fig:framework}.
	For a better illustration, $\bar{\mathcal{C}}_t$ denotes the target-only classes, and $\mathcal{C}=\mathcal{C}_s \cap \mathcal{C}_t$ denotes the classes existing in both domains.
	The ultimate goal is to recognize target samples belonging to $\mathcal{C}$ and reject target samples belonging to $\bar{\mathcal{C}_t}$ as ``unknown" (\ie, recognized as the $(K+1)$-th class, where $K=|\mathcal{C}_s|$ denotes the number of classes in the source domain).
	
	\subsection{Model training in the source domain}
	Although \cite{peng2020federated} offers a data-inaccessible solution via federated learning, the manner of gradient aggregation may be risky \cite{zhu2019deep} and it still needs the presence of source data during adaptation. 
	Instead, a trained source model could be provided to different target domains and would not leak the source data to our best knowledge.
	Then we mainly follow previous methods \cite{liang2020we,li2020model} and adopt the model transfer strategy to address the data privacy issues.
	As shown in Fig.~\ref{fig:UMAD} (a), we adopt a similar network architecture as \cite{liang2020we}, and the only difference lies in the classifier layer. 
	Denote by $g_s: \mathcal{X}_s \to \mathcal{R}^{d}$ the feature module within the network, and $h_v: \mathcal{R}^{d} \to \mathcal{R}^{K}, v=\{1,2\}$ represent two different classifiers taking features from $g_s$, respectively. 
	For simplicity, we denote $p^v(x_s)=h_v(g_s(x_s))$ as the $K$-dimensional soft-max class probability from the $v$-th classifier for an input $x_s$.
	The source model with two heads is learned by minimizing the following regularized co-training loss,
	\begin{equation}
		\begin{aligned}
			\mathcal{L}_{src} = -\mathbb{E}_{v, (x_s,y_s)\in \mathcal{D}_s} \sum\nolimits_{k} q_k \log p_{k}^{v}(x_s) + \lambda \|W_1^TW_2\|,
		\end{aligned}
	\end{equation}  
	where $p_{k}^{v}(x_s)$ denotes the $k$-th element of the soft-max output $p^{v}(x_s)$, and $q_k=(1-\alpha)\mathbbm{1}[y_s=k] + \alpha/K$ is the smoothed one-of-$K$ encoding of $y_s$, and $\alpha$ is the smoothing parameter empirically set to 0.1.
	Besides, $\lambda$ is determined based on validation set, and the latter orthogonal constraint \cite{chen2011co,saito2017asymmetric} ($W_v$ denotes the fc weight in $h_v, v=1,2$) is employed to expect each head to learn different features.
	
	\subsection{Model adaptation in the target domain}
	Before we transfer the trained source model $f_s$ in the target domain $\mathcal{D}_t$, we need to answer the following questions, (i) how to distinguish target samples belonging to $\mathcal{C}$ from those belonging to $\bar{\mathcal{C}}_t$, and (ii) how to mitigate the domain shift between $\mathcal{D}_s$ and $\mathcal{D}_t$ using samples from $\mathcal{C}$.
	For the first question, there are three common solutions to measure the uncertainty of target samples, \ie, confidence, entropy, and consistency.
	Different from the mixture of three criteria in \cite{fu2020learning}, we develop a new informative consistency score based on the two-head classifier, combining the advantages of both consistency and information entropy. 
	It simply measures the inner product distance between the soft-max probabilities from $h_1, h_2$,
	\begin{equation}
		\begin{aligned}
			\textbf{iscore}(x_t; g_t, h_1, h_2) = \textbf{w}_t = <p^{1}(x_t), \;p^{2}(x_t)>,
		\end{aligned}
		\label{eq:score}
	\end{equation} 
	where $<,>$ denotes the inner product.
	First, $\textbf{w}_t$ always lies in the range of $[0,1]$ and gets rid of the non-trivial normalization. 
	Second, when $p^{1}=p^{2}$, $\textbf{w}_t$ owns the same property as the collision entropy, which can measure the informative uncertainty.
	Third, the discrepancy between different classifiers has been proven effective for OOD detection \cite{yu2019unsupervised}.
	Hereafter, we term $\textbf{w}_t$ in Eq.~(\ref{eq:score}) as ``information consistency score".
	The larger $\textbf{w}_t$ is, the more likely $x_t$ belongs to $\mathcal{C}$ (known classes). 
	We can confirm unknown samples have smaller values than known samples in Fig.~\ref{fig:score}.
	We  prove in Fig.~\ref{fig:score} that unknown samples generally have smaller \textbf{iscore} values than known ones.
	
	\begin{figure}[!t]
		\begin{center}
			\includegraphics[width=0.44\textwidth, trim=0 0 0 0,clip]{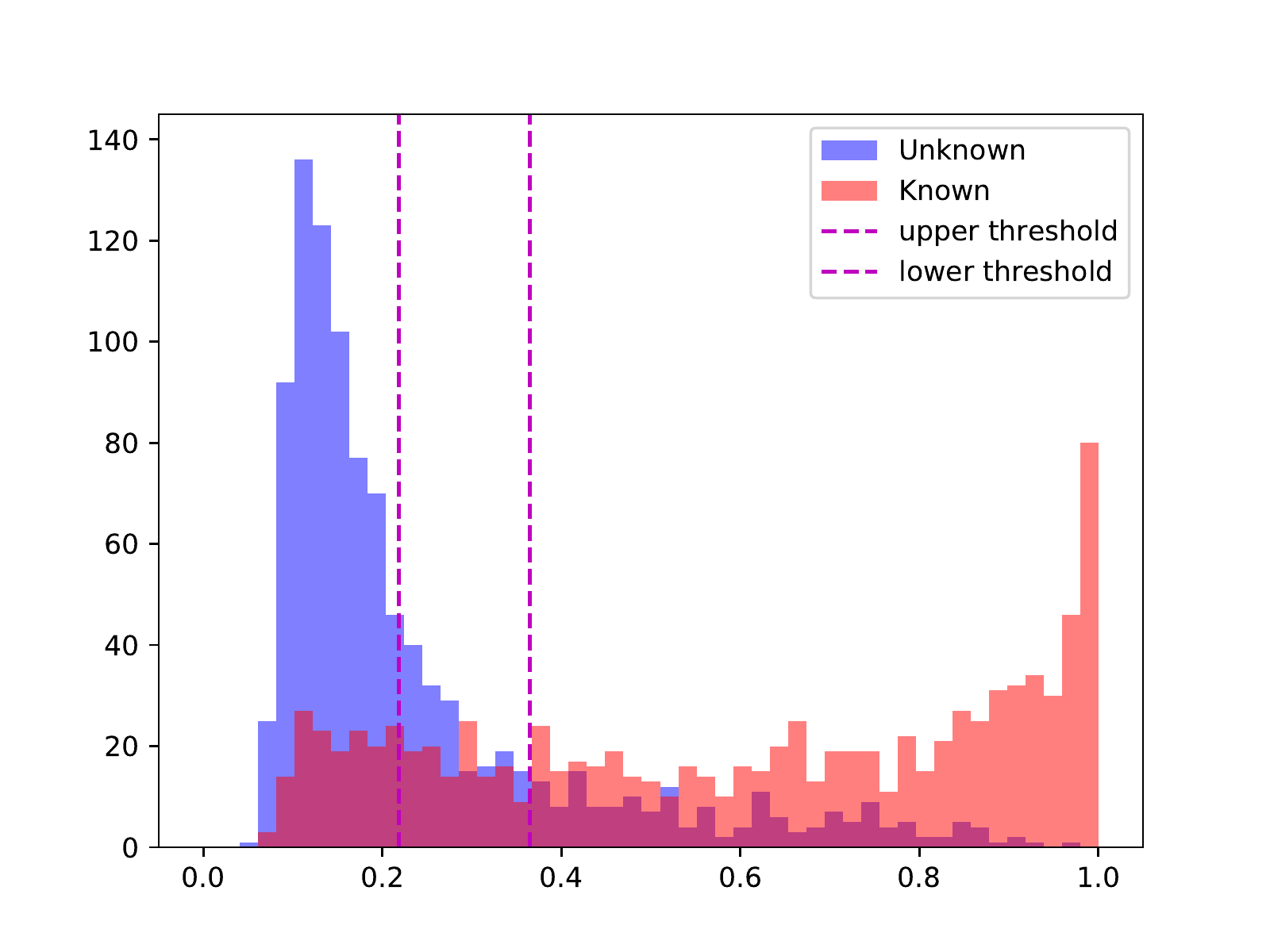}
			\put(-165,115){\textbf{$w_{0}-\rho$}}
			\put(-127,85){\textbf{$w_{0}+\rho$}}
			\caption{Histogram of the informative consistency score of all the target samples (\textbf{\color{red}known} \& \textbf{\color{blue}unknown}) through the source model.}
			\label{fig:score}
		\end{center}
	\end{figure}
	
	As can be shown from Fig.~\ref{fig:score}, despite the domain shift, samples with $\mathcal{C}$ (known) are expected to have large scores and samples from $\bar{\mathcal{C}}_t$ (unknown) are expected to have small scores.
	Therefore, we aim to utilize a threshold $w_0$ to distinguish known samples and unknown samples in the target domain.
	To avoid the sensitivity to $w_0$, we further introduce a slack margin $\rho$ and only consider the target samples beyond two thresholds $w_0-\rho$ and $w_0+\rho$ during adaptation.
	Since we are provided with the source model instead of the raw source data, it is hard to exploit prevailing paradigms like domain adversarial training and discrepancy minimization under the source data-free problem setting.
	Inspired by \cite{liang2020we}, we also freeze the classifier within the source model and try to merely learn the target-specific feature module by fine-tuning the source feature module.
	
	\begin{lemma}
		For a linear multi-class classifier $f_m$, $p$ is the soft-max output for the feature $z$. 
		Given a well-trained data set $Z_s=\{z|\ \|f_m(z)\|_2 \to 1\}$ and the distance from a data point to a set $D(z_t, Z_s)=\min_{z\in Z_s} d_{cosine}(z_t, z)$, then we have $\lim_{\|f_m(z_t)\|_2 \to 1} D(z_t, Z_s) \to 0$.
		\label{lemma}
	\end{lemma}
	\begin{proof}
		For the sake of simplicity, we consider a linear classifier without the bias term, \emph{i.e.}, $f_m(z)=\delta(W^Tz)$, where $W \in \mathcal{R}^{d\times K}$, $\delta(\cdot)$ denotes the soft-max function, and $d$ denotes the length of a feature vector $z$, and $K$ denotes the size of classes.
		The non-negative cosine distance is defined as $d_{cosine}=1-\frac{<x,y>}{\|x\|\cdot\|y\|}$.
		Denote $W=[w_1, w_2, \cdots, w_K]$, then we can obtain the following term,
		\begin{equation}
			\begin{aligned}
				\|f_m(z)\|_2 & = \|\delta(W^Tz)\|_2 = \sum\nolimits_j \left(\frac{\exp(w_j^Tz)}{\sum_k \exp(w_k^Tz)}\right)^2,\\
				\|f_m(z)\|_2 \to 1 & \Rightarrow \exists j\in [1,K],\; w_j^Tz \gg w_k^Tz, \;k\not = j, \\
				& \Rightarrow \exists j\in [1,K],\; d_{cosine}(z,w_j) \to 0.
			\end{aligned}
		\end{equation}
		Then the distance from $z_t$ to $Z_s$ can be computed as 
		\begin{equation}
			\begin{aligned}
				D(z_t, Z_s)&=\min_{z\in Z_s} d_{cosine}(z_t, z)\\
				&\leq \min_{z\in Z_s} d_{cosine}(z_t, w_j) + d_{cosine}(z, w_j),\\
				&\leq d_{cosine}(z_t, w_j) + \min_{z\in Z_s} d_{cosine}(z, w_j),\\
				\text{where} \; j &= \arg\min_k d_{cosine}(z_t, w_k).
			\end{aligned}
		\end{equation}
		Since the source data set contains data from each class, it is easy to find the later term $\min_{z\in Z_s} d_{cosine}(z, w_j)\to 0$.
		In this manner, the conclusion $D(z_t, Z_s)\to 0$ is arrived.
	\end{proof}
	
	According to Lemma \ref{lemma}, for ``known" samples with consistency scores higher than $w_0+\rho$, we devise a novel loss objective called localized mutual information $\mathcal{L}_{lmi}$ in the following, to encourage target features with diversified one-hot network outputs,
	\begin{equation}
	\begin{aligned}
		\mathcal{L}_{lmi} &= \mathbb{E}_{v, x_t\in \mathcal{X}_t^{+}} p^v(x_t)\log p^v(x_t) - D_{kl}(Q_t^v\: || \: \mathbf{\hat{Q}_t^v}), \\
		\mathcal{X}_t^{+} & = \{x_t | x_t \in X_t, \textbf{w}_t>w_0+\rho\}, \: Q_t^v =\mathbb{E}_{x_t\in \mathcal{X}_t^{+}} p^v(x_t), \\
	\end{aligned}
	\label{eq:known}
	\end{equation}  
	where $D_{kl}$ is the Kullback-Leibler divergence function, $Q_t^v$ denotes the batch-level mean probability of the $v$-th classifier, and $\hat{Q}_t^v$ is a $K$-dimensional uniform vector in \cite{liang2020we,liang2021source}.
	It is easy to find $\mathcal{L}_{lmi}$ considers the diversity and circumvents the trivial solution where all outputs collapse to a particular class. 
	However, for partial-set or open-partial-set UDA, it is not reasonable since it would force target samples to be wrongly recognized as classes that only exist in the source domain.
	To address this problem in class-imbalanced domain adaptation, \cite{li2020rethinking} proposes to use a moving average $\hat{Q}_t^v$ of $p(x_t)$. 
	However, for closed-set UDA, it relies on a large batch size during training to estimate an accurate global label distribution. 
	Besides, within a mini-batch, assuming the local label distribution is close to the global one is also not desirable.
	Instead, we propose a new operation called \textbf{flattening} as,
	\begin{equation}
		\begin{aligned}
			\text{Flatten}(p, T)_i = p_i^{T}/ \sum\nolimits_i p_i^{T}, \;0<T<1,
		\end{aligned}
		\label{eq:flatten}
	\end{equation}
	which is utilized to estimate the local label distribution.
	To avoid the influence of limited batch size and imbalanced label distribution, the batch-level mean probability $Q_t^v$ is required to be close to its flattened one in the following,
	\begin{equation}
		\begin{aligned}
			\mathbf{\hat{Q}_t^v} = \text{Flatten}(Q_t^v,T).
		\end{aligned}
	\end{equation}
	Note that $\mathcal{L}_{lmi}$ becomes the mutual information term when $T=0$, and the simplified entropy minimization term for $T=1$.
	In this paper, we empirically set this parameter to 0.1 for all open-set and open-partial-set UDA tasks.
	
	To tackle with ``unknown" samples with consistency scores smaller than $w_0-\rho$, we do not introduce an additional $(K+1)$-th class but employ the following entropic open-set loss \cite{dhamija2018reducing} to push the target samples far away from the source classification boundary,
	\begin{equation}
		\begin{aligned}
			\mathcal{L}_{unk} = -\mathbb{E}_{v,x_t\in \mathcal{X}_t^{-}} \sum\nolimits_{k} \frac{1}{K}\:\log p_{k}^{v}(x_t), 
		\end{aligned}
		\label{eq:unknown}
	\end{equation}
	where $\mathcal{X}_t^{-} = \{x_t | x_t \in X_t, \textbf{w}_t<w_0-\rho\}$, and $K$ denotes the size of source label set $\mathcal{C}_s$.
	The minimum of the loss $\mathcal{L}_{unk}$ for a sample $x_t$ is achieved when each soft-max scores $p^v(x_t)$ equals to the uniform vector.
	In that case, the consistency score becomes $w_t=1/K$ which is quite small.
	
	Combining these two different objectives in Eq.~(\ref{eq:known}) and Eq.~(\ref{eq:unknown}) for known and unknown samples together, the final objective of UMAD during target adaptation is
	\begin{equation}
		\begin{aligned}
			\mathcal{L}_{tgt} = \mathcal{L}_{unk} - \mathcal{L}_{lmi}.
		\end{aligned}
		\label{eq:final}
	\end{equation}
	
	\subsection{How to decide the threshold $w_0$?} 
	\label{sec:par}
	As the threshold parameter, $w_0$ is crucial for open-set learning including our UMAD approach, we further provide an intuitive strategy on how to select it.
	Assuming that samples in known classes $\mathcal{C}$ have confident predictions and samples in unknown classes $\bar{\mathcal{C}}_t$ have confused predictions, we can use Eq.~(\ref{eq:score}) to calculate the mean informative consistency score of unknown samples for $w_0$.
	But the target domain is completely unlabeled, it is impracticable to obtain such a score.
	Instead, we synthesize such ``negative" samples via \emph{mixup} \cite{zhang2018mixup}.
	Then the mean score of ``negative" samples is obtained by
	\begin{equation}
		\begin{aligned}
			w_u = \mathbb{E}_{x_i, x_j\in \mathcal{X}_t} \: \textbf{iscore}(0.5x_i + 0.5x_j ).
		\end{aligned}
		\label{eq:sys}
	\end{equation}
	Intuitively, for a pair of known samples with high confidences, the confidence after mixing will be halved.
	By contrast, for a pair of unknown samples, the confidence after mixing is believed to be similar.
	In other words, $w_u$ is much smaller than $\mathbb{E}_{x\in \mathcal{X}_t} \: \textbf{iscore}(x)$ if the number of samples in $\bar{\mathcal{C}}_t$ is small, and $w_u$ is relatively smaller with the number of samples in $\bar{\mathcal{C}}_t$ increases, indicating $w_u$ could be a suitable choice for $w_0$. 
	It corresponds with the assumption that ``\emph{mixing a known sample with another known sample gets an unknown sample while mixing an unknown sample with a known sample gets another unknown sample}". 
	Besides, we empirically set the slack margin $\rho=0.1*w_0$, its sensitivity analysis could be found in the experiment.
	
	\setlength{\tabcolsep}{3.0pt}
	\begin{table}[htb]
		\centering
		\caption{Details of class splits in each setting. $|\mathcal{C}_{s} \cap \mathcal{C}_{t}|$: \# shared classes across domains, $|\mathcal{C}_s - \mathcal{C}_{t}|$: \# private classes in the source domain, and $|\mathcal{C}_t - \mathcal{C}_{s}|$: \# private classes in the target domain.}
		\resizebox{0.45\textwidth}{!}{$
			\begin{tabular}{ccccc}
				\toprule
				Tasks & Datasets & $|\mathcal{C}_{s} \cap \mathcal{C}_{t}|$ & $|\mathcal{C}_s - \mathcal{C}_{t}|$ & $|\mathcal{C}_t - \mathcal{C}_{s}|$ \\
				\midrule
				\multirow{3}{*}{OSDA~\cite{liu2019separate}} & Office~\cite{saenko2010adapting} & 10 & 0 & 11 \\
				& Office-Home~\cite{venkateswara2017deep} & 25 & 0 & 40 \\
				& VisDA-C~\cite{peng2017visda} & 6 & 0 & 6 \\
				\midrule
				\multirow{4}{*}{OPDA~\cite{you2019universal}} & Office~\cite{saenko2010adapting} & 10 & 10 & 11\\
				& Office-Home~\cite{venkateswara2017deep} & 10 & 5 & 50\\
				& VisDA-C~\cite{peng2017visda} & 6  & 3  & 3 \\
				& DomainNet~\cite{peng2019moment} & 150 & 50 & 145 \\  
				\bottomrule
			\end{tabular}
			$}
		\label{tab:classsplit}
	\end{table}

	\setlength{\tabcolsep}{3.0pt}
	\begin{table*}[!htb]
		\centering
		\caption{HOS (\%) averaged over three runs of each method on \textbf{Office} and \textbf{VisDA-C} for OSDA. (Best in \textbf{\color{blush}red} and second best in \textbf{\color{bblue}blue}.)} 
		\label{tab:office-oda}
		\resizebox{0.75\textwidth}{!}{$
			\begin{tabular}{lccc|cccccca|a}
				\toprule
				Method & SDF & OSDA & OPDA & A $\to$ D & A $\to$ W & D $\to$ A & D $\to$ W & W $\to$ A & W $\to$ D & Avg. & VisDA-C \\
				\midrule
				OSBP \cite{saito2018open} & $\times$ & $\checkmark$ & $\times$ & 78.0 & 75.9 & 67.0 & 76.3 & 69.0 & 78.9 & 74.2 & 46.9 \\
				UAN \cite{you2019universal} & $\times$ & $\times$ & $\checkmark$ & 54.2 & 57.4 & 73.7 & 75.2 & 59.8 & 67.6 & 64.6 & 50.8 \\
				CMU \cite{fu2020learning}  & $\times$ & $\times$ & $\checkmark$ & 71.6 & 70.5 & 80.2 & 81.2 & 70.8 & 70.8 & 74.2 & 24.1 \\
				ROS \cite{bucci2020effectiveness} & $\times$ & $\checkmark$ & $\times$ & 65.8 & 71.7 & \textbf{\color{bblue}87.2} & 94.8 & 82.0 & \textbf{\color{bblue}98.2} & 83.3 & 50.1 \\
				DANCE \cite{saito2021universal} & $\times$ & $\checkmark$ & $\checkmark$ & 82.0 & 74.7 & 68.0 & 82.1 & 52.2 & 82.5 & 73.6 & 59.7 \\
				DCC$^\dagger$ \cite{li2021domain} & $\times$ & $\checkmark$ & $\checkmark$ & 58.3 & 54.8 & 67.2 & 89.4 & 85.3 & 80.9 & 72.6 & 70.7 \\
				OVANet \cite{saito2021ovanet} & $\times$ & $\checkmark$ & $\checkmark$ & \textbf{\color{bblue}89.4} & \textbf{\color{bblue}87.6} & 85.2 & \textbf{\color{blush}97.2} & 87.8 & 97.9 & \textbf{\color{bblue}90.8} & 56.1 \\
				SHOT \cite{liang2020we} & $\checkmark$ & $\checkmark$ & $\times$ & 80.2 & 71.6 & 64.3 & 93.1 & 64.0 & 91.8 & 77.5 & 28.1 \\
				Inheritune$^\dagger$ \cite{kundu2020towards} & $\checkmark$ & $\checkmark$ & $\times$ & 78.0 & 81.4 & 83.1 & 92.2 & \textbf{\color{blush}91.3} & \textbf{\color{blush}99.7} & 87.6 & 74.8 \\
				OSHT-SC$^\dagger$ \cite{feng2021open} & $\checkmark$ & $\checkmark$ & $\times$ & \textbf{\color{blush}91.3} & \textbf{\color{blush}92.4} & \textbf{\color{blush}90.8} & \textbf{\color{bblue}95.2} & \textbf{\color{bblue}89.6} & 96.0 & \textbf{\color{blush}92.5} & \textbf{\color{bblue}78.6} \\
				UMAD & $\checkmark$ & $\checkmark$ & $\checkmark$ & 88.5 & 84.4 & 86.8 & 95.0 & 88.2 & 95.9 & 89.8 & \textbf{\color{blush}80.2} \\
				\hline
			\end{tabular}
			$}
	\end{table*}

	\section{Experiments}	
	\subsection{Setup}
	\textbf{Datasets.} 
	\textbf{Office} \cite{saenko2010adapting} is a widely-used DA benchmark that consists of three subsets, \ie, Amazon (A), Dslr (D), and Webcam (W). Each subset contains 31 object classes under the office environment. 
	\textbf{Office-Home} \cite{venkateswara2017deep} is another popular benchmark that consists of four subsets, \ie, Artistic images (Ar), Clip Art (Cl), Product images (Pr), and Real-world images (Re). Each subset contains 65 object classes.
	\textbf{VisDA-C} \cite{peng2017visda} is a large-scale testbed, consisting of 12 object classes in two subsets, \ie, real images and synthesis data. Typically, only the synthesis-to-real transfer task is investigated.
	\textbf{DomainNet} \cite{peng2019moment} is the largest domain adaptation dataset with about 0.6 million images, and we follow \cite{fu2020learning} to use three subsets, \ie, Painting (P), Real (R), and Sketch (S).
	Details of classes split is listed in Table~\ref{tab:classsplit}.	
		
	\begin{table*}[!htb]
		\caption{HOS (\%) averaged over three runs of each method on \textbf{Office-Home} for OSDA.} 
		\label{tab:home-oda}
		\resizebox{0.99\textwidth}{!}{$
			\begin{tabular}{lccc|cccccccccccca}
				\toprule
				Method & SDF & OSDA & OPDA & Ar$\to$Cl & Ar$\to$Pr & Ar$\to$Re & Cl$\to$Ar & Cl$\to$Pr & Cl$\to$Re & Pr$\to$Ar & Pr$\to$Cl & Pr$\to$Re & Re$\to$Ar & Re$\to$Cl & Re$\to$Pr &Avg. \\
				\midrule
				OSBP \cite{saito2018open} & $\times$ & $\checkmark$ & $\times$ & 53.4 & 62.6 & 63.0 & 56.3 & 60.1 & 66.5 & 55.8 & 51.7 & 61.9 & 57.4 & 48.6 & 59.7 & 58.1 \\
				UAN \cite{you2019universal} & $\times$ & $\times$ & $\checkmark$ & 34.7 & 22.4 & 9.4 & 38.9 & 22.9 & 21.8 & 47.4 & 39.7 & 30.9 & 34.4 & 35.8 & 22.0 & 30.0 \\
				CMU \cite{fu2020learning} & $\times$ & $\times$ & $\checkmark$ & 55.0 & 57.0 & 59.0 & 59.3 & 58.2 & 60.6 & 59.2 & 51.3 & 61.2 & 61.9 & 53.5 & 55.3 & 57.6 \\
				ROS \cite{bucci2020effectiveness} & $\times$ & $\checkmark$ & $\times$ & 53.4 & 66.3 & \textbf{\color{bblue}73.9} & 54.9 & 62.3 & 66.9 & 58.7 & 50.4 & \textbf{\color{bblue}72.6} & 65.1 & 52.6 & \textbf{\color{bblue}72.0} & 62.4 \\
				DANCE \cite{saito2021universal} & $\times$ & $\checkmark$ & $\checkmark$ & 6.5 & 9.0 & 9.9 & 20.4 & 10.4 & 9.2 & 28.4 & 12.8 & 12.6 & 14.2 & 7.9 & 13.2 & 12.9 \\
				DCC$^\dagger$ \cite{li2021domain} & $\times$ & $\checkmark$ & $\checkmark$ & \textbf{\color{blush}66.7} & \textbf{\color{bblue}67.9} & 66.7 & 48.6 & \textbf{\color{bblue}66.5} & 63.7 & 54.9 & \textbf{\color{bblue}53.7} & 70.5 & 62.1 & \textbf{\color{bblue}58.2} & \textbf{\color{blush}72.4} & 61.9 \\
				OVANet \cite{saito2021ovanet} & $\times$ & $\checkmark$ & $\checkmark$ & 58.4 & 66.3 & 69.3 & \textbf{\color{bblue}60.3} & 65.1 & \textbf{\color{bblue}67.2} & 58.8 & 52.4 & 68.7 & \textbf{\color{blush}67.6} & \textbf{\color{blush}58.6} & 66.6 & \textbf{\color{bblue}63.3} \\
				SHOT \cite{liang2020we} & $\checkmark$ & $\checkmark$ & $\times$ & 37.7 & 41.8 & 48.4 & 56.4 & 39.8 & 40.9 & \textbf{\color{bblue}60.0} & 41.5 & 49.7 & 61.8 & 41.4 & 43.6 & 46.9 \\
				OSHT-SC$^\dagger$ \cite{feng2021open} & $\checkmark$ & $\checkmark$ & $\times$ & 40.9 & 32.3 & 40.8 & 30.6 & 23.8 & 24.2 & 49.8 & 31.8 & 40.2 & 31.3 & 46.8 & 46.1 & 36.6 \\
				UMAD & $\checkmark$ & $\checkmark$ & $\checkmark$ & \textbf{\color{bblue}59.2} & \textbf{\color{blush}71.8} & \textbf{\color{blush}76.6} & \textbf{\color{blush}63.5} & \textbf{\color{blush}69.0} & \textbf{\color{blush}71.9} & \textbf{\color{blush}62.5} & \textbf{\color{blush}54.6} & \textbf{\color{blush}72.8} & \textbf{\color{bblue}66.5} & 57.9 & 70.7 & \textbf{\color{blush}66.4} \\
				\bottomrule
			\end{tabular}
			$}
	\end{table*}
	
	\textbf{Baseline methods.}
	Using the code from each corresponding author, we implement several methods for comparisons with our method, including OSBP~\cite{saito2018open}, UAN~\cite{you2019universal}, CMU~\cite{fu2020learning}, ROS~\cite{bucci2020effectiveness}, DANCE~\cite{saito2020universal}, OVANet~\cite{saito2021ovanet}, DCC~\cite{li2021domain}, and SHOT~\cite{liang2020we}. 
	We adopt class-balanced source sampling on all these methods.
	We also compare our method with Inheritune~\cite{kundu2020towards}, USFDA~\cite{kundu2020universal}, and OSHT-SC~\cite{feng2021open} using available results.
	Note that, as for the traditional data-dependent methods, OSBP and ROS are proposed for open-set UDA (OSDA), UAN and CMU and DCC are proposed for open-partial-set UDA (OPDA), and DANCE and OVANet are recently proposed to solve both UDA scenarios.  
	SHOT is a popular source data-free method for closed-set UDA but can also be used as a baseline for open-set UDA, Inheritune and OSHT-SC are methods for open-set UDA, and USFDA is proposed to solve open-partial-set UDA. 
	For comprehensive comparisons, we classify the above methods according to their applicable UDA scenarios, \ie, OSDA and OPDA, and whether they are source data-free (SDF). 
	Methods with $^\dagger$ denotes results copied from the original papers or OVANet \cite{saito2021ovanet}.
	Further, full results of DCC \cite{li2021domain} are kindly provided by the authors.
	
	\textbf{Implementation details.}
	All the experiments are conducted via \textbf{Pytorch} \cite{paszke2019pytorch}.
	We adopt mini-batch SGD to learn the feature encoder by fine-tuning from the ImageNet pre-trained \textbf{ResNet-50} model with the learning rate 0.001, and new layers (bottleneck layer and classification layer) from scratch with the learning rate of 0.01.
	Batch Normalization \cite{ioffe2015batch} is employed in the bottleneck layer.
	For training in both domains, we set the maximum number of iteration as 3,000, and use the suggested training settings in \cite{liang2020we}, including learning rate scheduler, momentum (0.9), weight decay (1e-3), bottleneck size (256), and batch size (64).
	
	\textbf{Evaluation metrics.}
	Generally, we report the mean value over three random trials. 
	For both open-set and open-partial-set UDA tasks, we report the HOS score \cite{fu2020learning,bucci2020effectiveness}, which is a harmonic mean of accuracy over known samples and accuracy over unknown samples, in the main text.
	Denote $\text{acc}_{kn}$ the average per-class accuracy over the known classes and $\text{acc}_{ukn}$ the accuracy of the unknown class. 
	HOS score can be easily obtained by
	\begin{equation}
		HOS=2\times \frac{\text{acc}_{kn}\times \text{acc}_{ukn}}{\text{acc}_{kn}+\text{acc}_{ukn}}.
		\label{eq:hos}
	\end{equation} 
	The HOS score is considered fair to balance the rejection of unknown samples and the recognition of known samples.

	\setlength{\tabcolsep}{4.0pt}
	\begin{table*}[!htb]
		\centering
		\caption{HOS (\%) averaged over three runs of each method on \textbf{Office}, \textbf{VisDA-C}, and \textbf{DomainNet} for OPDA.} 
		\label{tab:office-opda}
		\resizebox{0.99\textwidth}{!}{$
			\begin{tabular}{lccc|cccccca|a|cccccca}
				\toprule
				Method & SDF & OSDA & OPDA & A$\to$D & A$\to$W & D$\to$A & D$\to$W & W$\to$A & W$\to$D & Avg. & VisDA-C & P$\to$R & P$\to$S & R$\to$P & R$\to$S & S$\to$P & S$\to$R & Avg. \\
				\midrule
				OSBP \cite{saito2018open} & $\times$ & $\checkmark$  & $\times$ & 61.7 & 59.5 & 62.5 & 62.4 & 60.6 & 65.1 & 62.0 & 37.7 & 52.2 & 35.0 & 46.5 & 35.8 & 38.6 & 52.1 & 43.4 \\
				UAN \cite{you2019universal} & $\times$ & $\times$ & $\checkmark$ & 52.0 & 53.2 & 70.8 & 76.0 & 64.8 & 64.1 & 63.5 & 34.8 & 0.1 & 28.5 & 37.7 & 31.4 & 31.4 & 12.2 & 23.6 \\
				CMU \cite{fu2020learning} & $\times$ & $\times$ & $\checkmark$ & 76.9 & 74.1 & 82.8 & 88.5 & 79.2 & 85.2 & 81.1 & 32.9 & 50.5 & 43.9 & 49.4 & \textbf{\color{bblue}43.5} & 44.0 & 48.1 & 46.6 \\
				ROS~\cite{bucci2020effectiveness} & $\times$ & $\checkmark$ & $\times$ & 29.8 & 26.8 & \textbf{\color{bblue}86.4} & 86.6 & 83.9 & \textbf{\color{bblue}96.0} & 68.3 & 30.3 & 20.5 & 30.0 & 36.9 & 28.7 & 19.9 & 23.2 & 26.5 \\
				DANCE~\cite{saito2021universal} & $\times$ & $\checkmark$ & $\checkmark$ & 78.8 & 72.0 & 78.8 & \textbf{\color{bblue}91.6} & 72.3 & 91.7 & 80.9 & 3.9& 38.8 & 43.8 & 48.1 & \textbf{\color{blush}43.8} & 39.4 & 20.9 & 39.1 \\
				DCC$^\dagger$ \cite{li2021domain} & $\times$ & $\checkmark$ & $\checkmark$ & \textbf{\color{blush}88.5} & 78.5 & 70.2 & 79.3 & 75.9 & 88.6 & 80.2 & 43.0 & \textbf{\color{bblue}56.9} & 43.7 & \textbf{\color{bblue}50.3} & 43.3 & \textbf{\color{bblue}44.9} & \textbf{\color{bblue}56.1} & \textbf{\color{bblue}49.2} \\
				OVANet \cite{saito2021ovanet} & $\times$ & $\checkmark$ & $\checkmark$ & 84.0 & \textbf{\color{bblue}79.4} & 75.1 & \textbf{\color{blush}95.8} & 82.1 & 95.8 & \textbf{\color{bblue}85.4} & \textbf{\color{bblue}44.7} & 55.8 & \textbf{\color{blush}45.5} & \textbf{\color{blush}51.0} & 43.2 & \textbf{\color{blush}46.7} & \textbf{\color{blush}56.7} & \textbf{\color{blush}49.8} \\
				USFDA$^\dagger$ \cite{kundu2020universal} & $\checkmark$ & $\times$ & $\checkmark$ & \textbf{\color{bblue}85.5} & \textbf{\color{blush}79.8} & 83.2 & 90.6 & \textbf{\color{bblue}88.7} & 81.2 & 84.8 & - & - & - & - & - & - & - & -\\
				SHOT \cite{liang2020we} & $\checkmark$ & $\checkmark$ & $\times$ & 73.5 & 67.2 & 59.3 & 88.3 & 77.1 & 84.4 & 74.9 & 44.0 & 35.0 & 30.8 & 37.2 & 28.3 & 31.9 & 32.2 & 32.6 \\
				UMAD & $\checkmark$ & $\checkmark$ & $\checkmark$ & 
				79.1 & 77.4 & \textbf{\color{blush}87.4} & 90.7 & \textbf{\color{blush}90.4} & \textbf{\color{blush}97.2} & \textbf{\color{blush}87.0} & \textbf{\color{blush}58.3} & \textbf{\color{blush}59.0} & \textbf{\color{bblue}44.3} & 50.1 & 42.1 & 32.0 & 55.3 & 47.1 \\
				\bottomrule
			\end{tabular}
			$}
	\end{table*}
	
	\setlength{\tabcolsep}{2.0pt}
	\begin{table*}[]
		\centering
		\caption{HOS (\%) averaged over three runs of each method on \textbf{Office-Home} for OPDA.} 
		\label{tab:home-opda}
		\resizebox{0.99\textwidth}{!}{$
			\begin{tabular}{lccc|cccccccccccca}
				\toprule
				Method & SDF & OSDA & OPDA & Ar$\to$Cl & Ar$\to$Pr & Ar$\to$Re & Cl$\to$Ar & Cl$\to$Pr & Cl$\to$Re & Pr$\to$Ar & Pr$\to$Cl & Pr$\to$Re & Re$\to$Ar & Re$\to$Cl & Re$\to$Pr &Avg. \\
				\midrule
				OSBP \cite{saito2018open} & $\times$ & $\checkmark$ & $\times$ & 51.0 & 55.2 & 71.6 & 51.4 & 48.3 & 61.2 & 55.5 & 50.7 & 63.4 & 52.6 & 49.7 & 54.3 & 55.4 \\
				UAN \cite{you2019universal} & $\times$ & $\times$ & $\checkmark$ & 38.7 & 32.1 & 30.6 & 45.2 & 32.8 & 35.1 & 54.9 & 42.0 & 47.0 & 51.5 & 50.8 & 43.9 & 42.1 \\
				CMU \cite{fu2020learning} & $\times$ & $\times$ & $\checkmark$ & 57.3 & 59.9 & 66.9 & 64.4 & 59.0 & 64.0 & 65.6 & \textbf{\color{bblue}56.5} & 68.6 & 67.2 & 60.2 & 65.1 & 62.9 \\
				ROS \cite{bucci2020effectiveness} & $\times$ & $\checkmark$ & $\times$ & 54.0 & \textbf{\color{blush}77.7} & \textbf{\color{blush}85.3} & 62.1 & \textbf{\color{bblue}71.0} & \textbf{\color{blush}76.4} & 68.8 & 52.4 & \textbf{\color{blush}83.2} & 71.6 & 57.8 & \textbf{\color{blush}79.2} & 70.0 \\
				DANCE \cite{saito2021universal} & $\times$ & $\checkmark$ & $\checkmark$ & 34.1 & 23.9 & 38.3 & 46.7 & 21.6 & 35.4 & 58.2 & 47.5 & 39.4 & 32.8 & 38.3 & 43.1 & 38.3 \\
				DCC$^\dagger$ \cite{li2021domain} & $\times$ & $\checkmark$ & $\checkmark$ & \textbf{\color{blush}80.8} & 73.8 & 80.8 & 63.7 & \textbf{\color{blush}71.7} & 69.3 & \textbf{\color{blush}73.5} & 53.6 & 80.2 & \textbf{\color{bblue}74.4} & 57.1 & 76.8 & 69.8 \\
				OVANet \cite{saito2021ovanet}  & $\times$ & $\checkmark$ & $\checkmark$ & 59.7 & \textbf{\color{bblue}76.9} & 80.0 & \textbf{\color{bblue}68.8} & 69.1 & \textbf{\color{bblue}76.2} & \textbf{\color{bblue}69.6} & \textbf{\color{blush}56.9} & \textbf{\color{bblue}81.0} & \textbf{\color{blush}75.5} & \textbf{\color{blush}62.0} & \textbf{\color{bblue}78.6} & \textbf{\color{blush}71.2} \\
				SHOT \cite{liang2020we}  & $\checkmark$ & $\checkmark$ & $\times$ & 32.9 & 29.5 & 39.6 & 56.8 & 30.1 & 41.1 & 54.9 & 35.4 & 42.3 & 58.5 & 33.5 & 33.3 & 40.7 \\
				UMAD & $\checkmark$ & $\checkmark$ & $\checkmark$ & \textbf{\color{bblue}61.1} & 76.3 & \textbf{\color{bblue}82.7} & \textbf{\color{blush}70.7} & 67.7 & 75.7 & 64.4 & 55.7 & 76.3 & 73.2 & \textbf{\color{bblue}60.4} & 77.2 & \textbf{\color{bblue}70.1} \\
				\bottomrule
			\end{tabular}
			$}
	\end{table*}

	\subsection{Results}
	\textbf{Open-set UDA.} 
	As for open-set UDA (OSDA), we compare our method with source data-free OSDA methods including SHOT, Inheritune, and OSHT-SC and make comparisons with data-dependent methods including universal methods like DANCE and OVANet, OSDA methods like OSBP and ROS, and OPDA methods like UAN and CMU. 
	We report the results on \textbf{Office} and \textbf{VisDA-C} in Table~\ref{tab:office-oda} and the results on \textbf{Office-Home} in Table~\ref{tab:home-oda}. 
	As shown in Table~\ref{tab:office-oda} and Table~\ref{tab:home-oda}, on OSDA tasks, methods designed for OPDA, \ie, UAN and CMU, generally underperform methods taking into account the OSDA scenario like ROS and OVANet. 
	This observation demonstrates the motivation of our work, \ie, the scenario-tailored UDA method tends to be suffering under other DA scenarios and it is desirable to devise universal methods for DA with unknown target samples.
	Making comparisons among all methods, we find that our method achieves the best on two benchmarks, \ie, 66.3\% on the challenging 12-task benchmark \textbf{Office-Home} and 80.2\% on the large-scale benchmark \textbf{VisDA-C}, and only underperforms the data-dependent method OVANet and the source data-free OSDA method OSHT-SC on the standard benchmark \textbf{Office}, which verifies the effectiveness of our method. Furthermore, our method is the only one that consistently achieves the leading performance on all three benchmarks, unlike other state-of-the-art methods. Specifically, the best data-dependent universal DA method OVANet only achieves a HOS score of 56.1\% on \textbf{VisDA-C} and the state-of-the-art source data-free OSDA method OSHT-SC only obtains a HOS score of 36.6\% on \textbf{Office-Home}, both of which are far worse than the corresponding HOS scores of our method. 	
	
	\begin{figure*}[!htb]
		\centering
		\small
		\setlength\tabcolsep{1mm}
		\renewcommand\arraystretch{0.1}
		\begin{tabular}{cccc}
			\includegraphics[width=0.24\linewidth,trim={3.3cm 9.2cm 4.2cm 9.8cm}, clip]{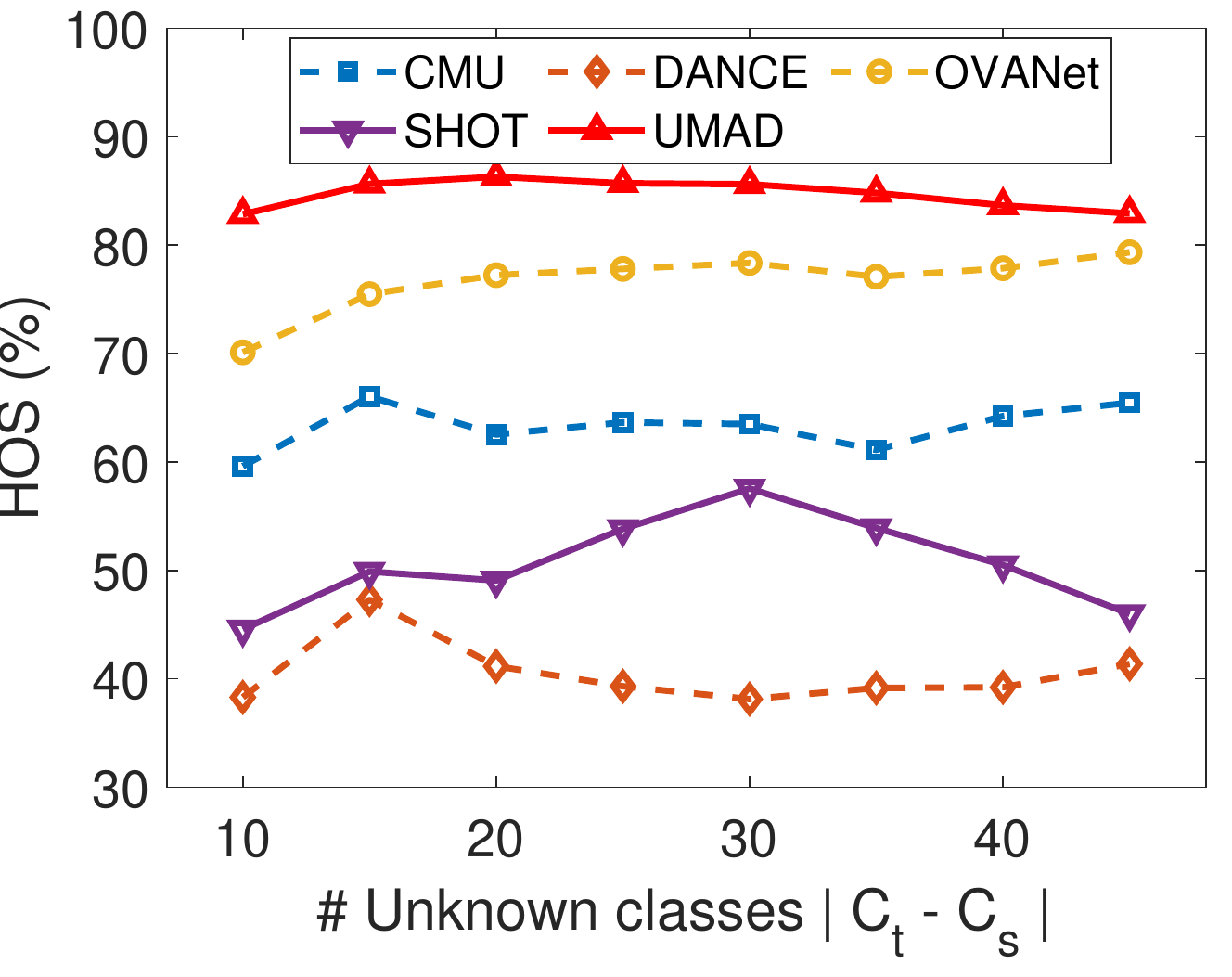} &
			\includegraphics[width=0.24\linewidth,trim={3.3cm 9.2cm 4.2cm 9.8cm}, clip]{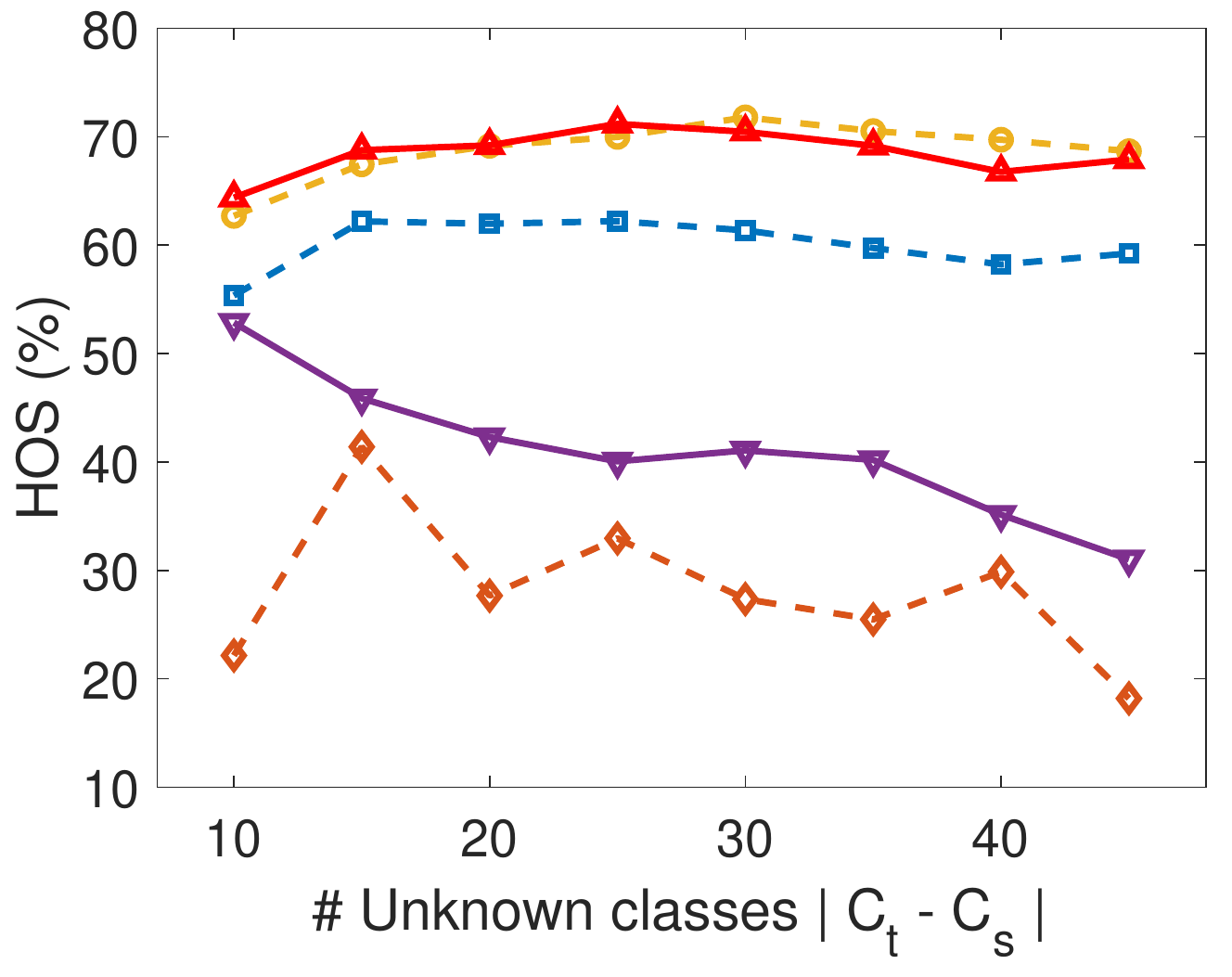} & 
			\includegraphics[width=0.24\linewidth,trim={3.3cm 9.2cm 4.2cm 9.8cm}, clip]{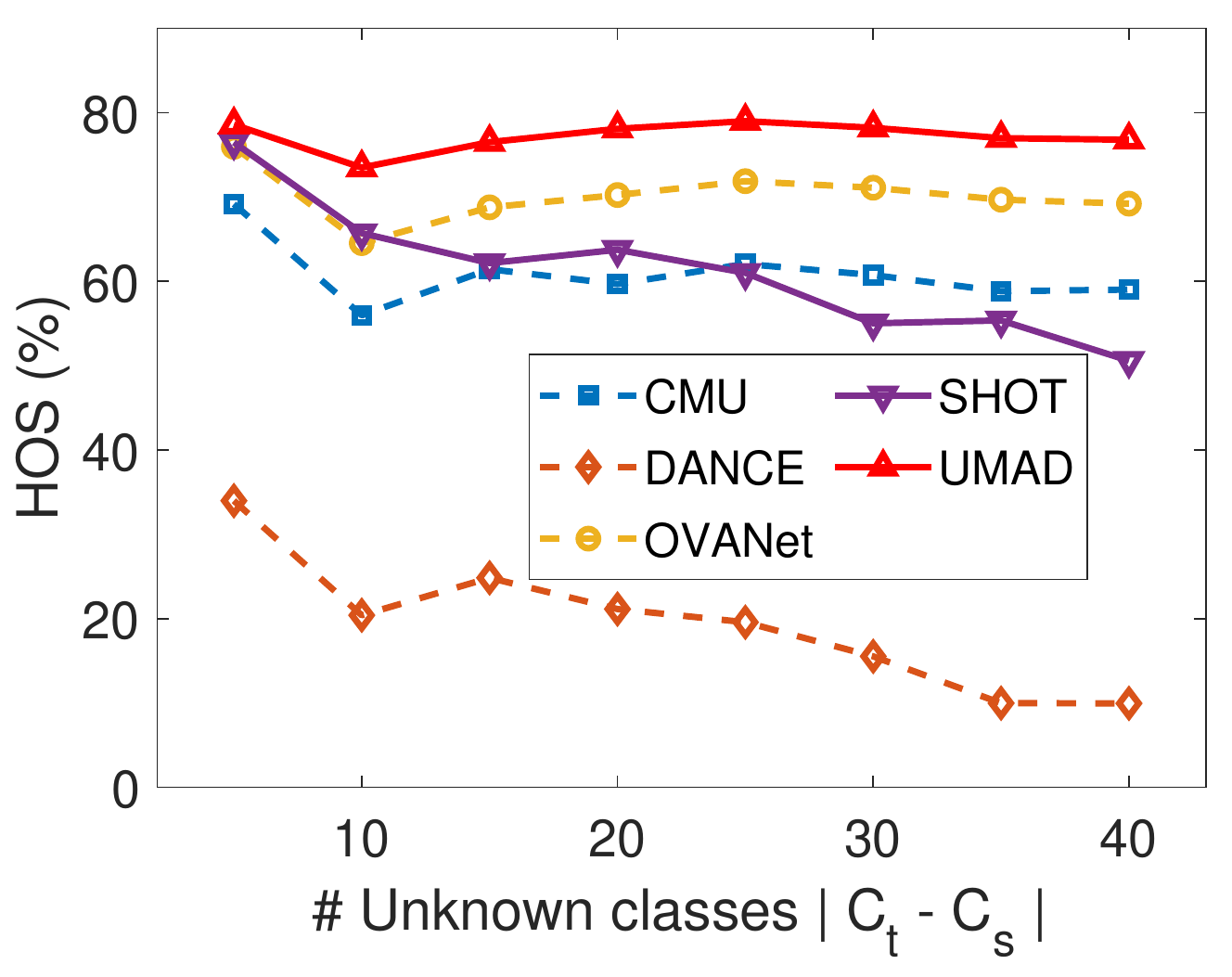} &
			\includegraphics[width=0.24\linewidth,trim={3.3cm 9.2cm 4.2cm 9.8cm}, clip]{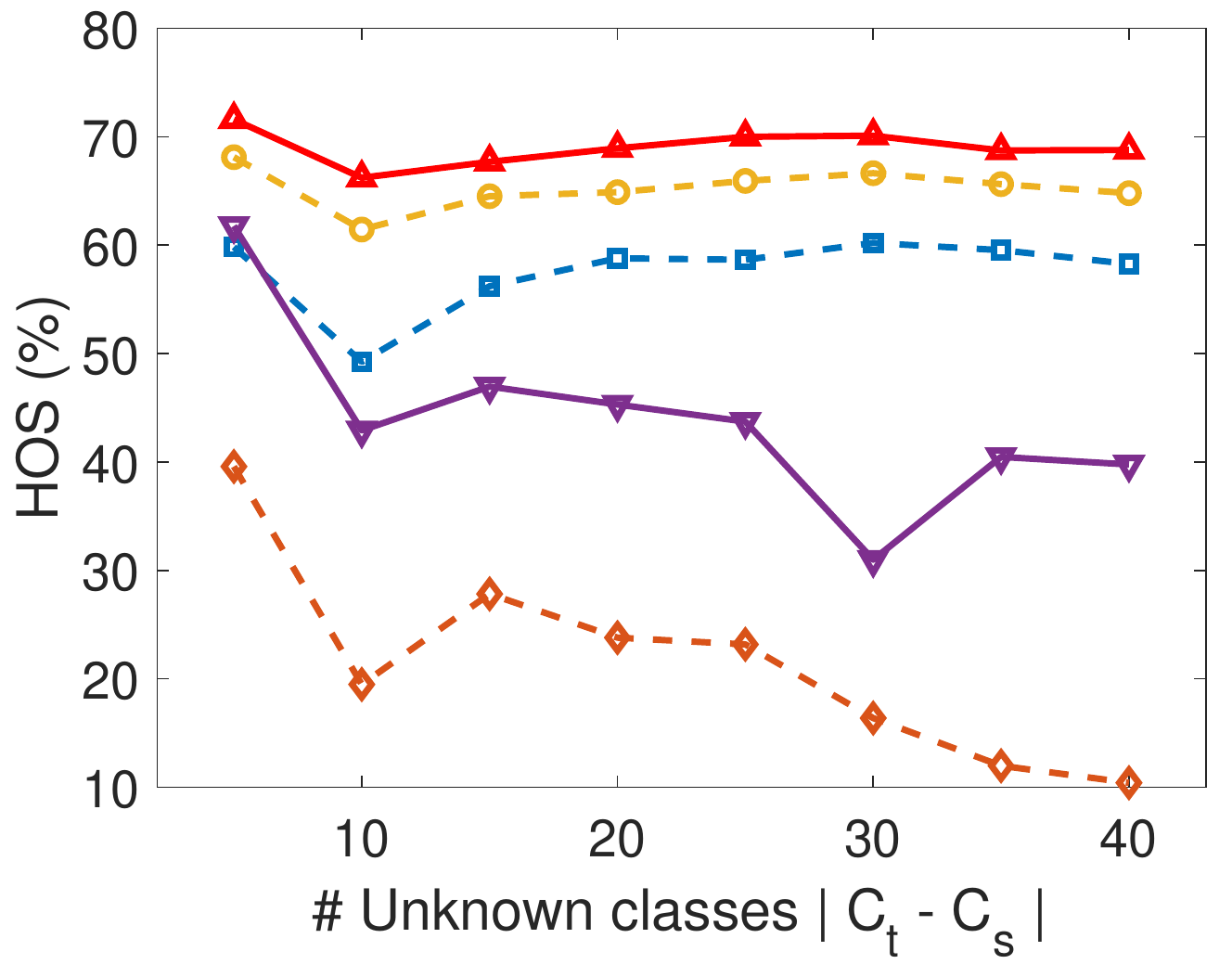} \\
			~\\
			(a) Ar $\to$ Re on OPDA & (b) Cl $\to$ Pr on OPDA & (c) Ar $\to$ Re on OSDA & (d) Cl $\to$ Pr on OSDA
		\end{tabular}
		\caption{\textbf{HOS (\%) of open-partial domain adaptation (OPDA) and open-set domain adaptation (OSDA)}. We vary the number of unknown classes (\ie, $\|\mathcal{C}_t-\mathcal{C}_s\|$) of four different transfer tasks on the \textbf{Office-Home} dataset.}
		\label{fig:par_num}
	\end{figure*} 
	
	\textbf{Open-partial-set UDA.} As for open-partial-set UDA (OPDA), similarly, we compare our methods with various scenario-specific or universal UDA methods. We report the results on \textbf{Office}, \textbf{VisDA-C}, and \textbf{DomainNet} in Table~\ref{tab:office-opda} and the results on \textbf{Office-Home} in Table~\ref{tab:home-opda}. As shown in Table~\ref{tab:office-opda} and Table~\ref{tab:home-opda}, we have a similar observation that OPDA methods tend to outperform OSDA methods on various OPDA benchmarks, \eg, CMU significantly beats the state-of-the-art OSDA method ROS on 3 out of all 4 benchmarks. Our method still achieves consistently leading performance on all benchmarks. 
	To be specific, our method achieves the best on two benchmarks, \ie, 87.0\% on \textbf{Office} and 58.3\% on \textbf{VisDA-C}, the third best HOS score of 48.0\% on the largest DA benchmark \textbf{DomainNet}, and the second best HOS score of 69.4\% on \textbf{Office-Home}.
	Besides, although slightly falling behind the state-of-the-art data-dependent universal UDA method OVANet on \textbf{DomainNet} and \textbf{Office-Home}, our method significantly outperforms OVANet on \textbf{VisDA-C} without accessing source data. 
	
	\textbf{Summary.} Comparing the results for OSDA and OPDA, we find that universal UDA methods like OVANet and UMAD are able to handle UDA scenarios with unknown category shifts, which makes them stand out among various scenario-specific UDA methods. 
	Besides the stable and superior performance, UMAD is more practical, in terms of better privacy-protection and lighter data transmission, than data-dependent methods with access to raw source data.

	\begin{figure*}[!htb]
		\centering
		\small
		\setlength\tabcolsep{1mm}
		\renewcommand\arraystretch{0.1}
		\begin{tabular}{cccc}
			\includegraphics[width=0.245\linewidth,trim={3.3cm 9.2cm 4.2cm 9.8cm}, clip]{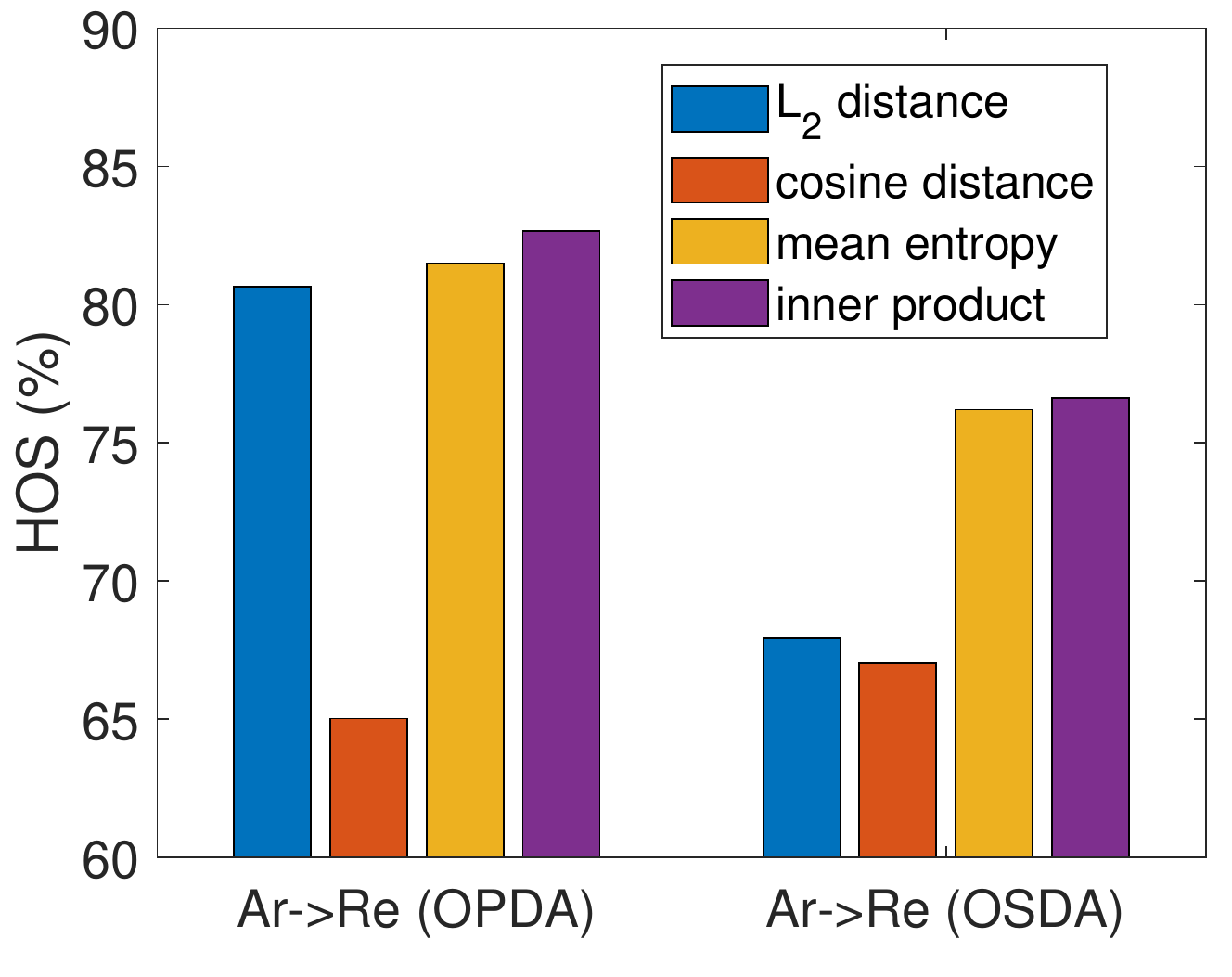} &
			\includegraphics[width=0.245\linewidth,trim={3.3cm 9.2cm 4.2cm 9.8cm}, clip]{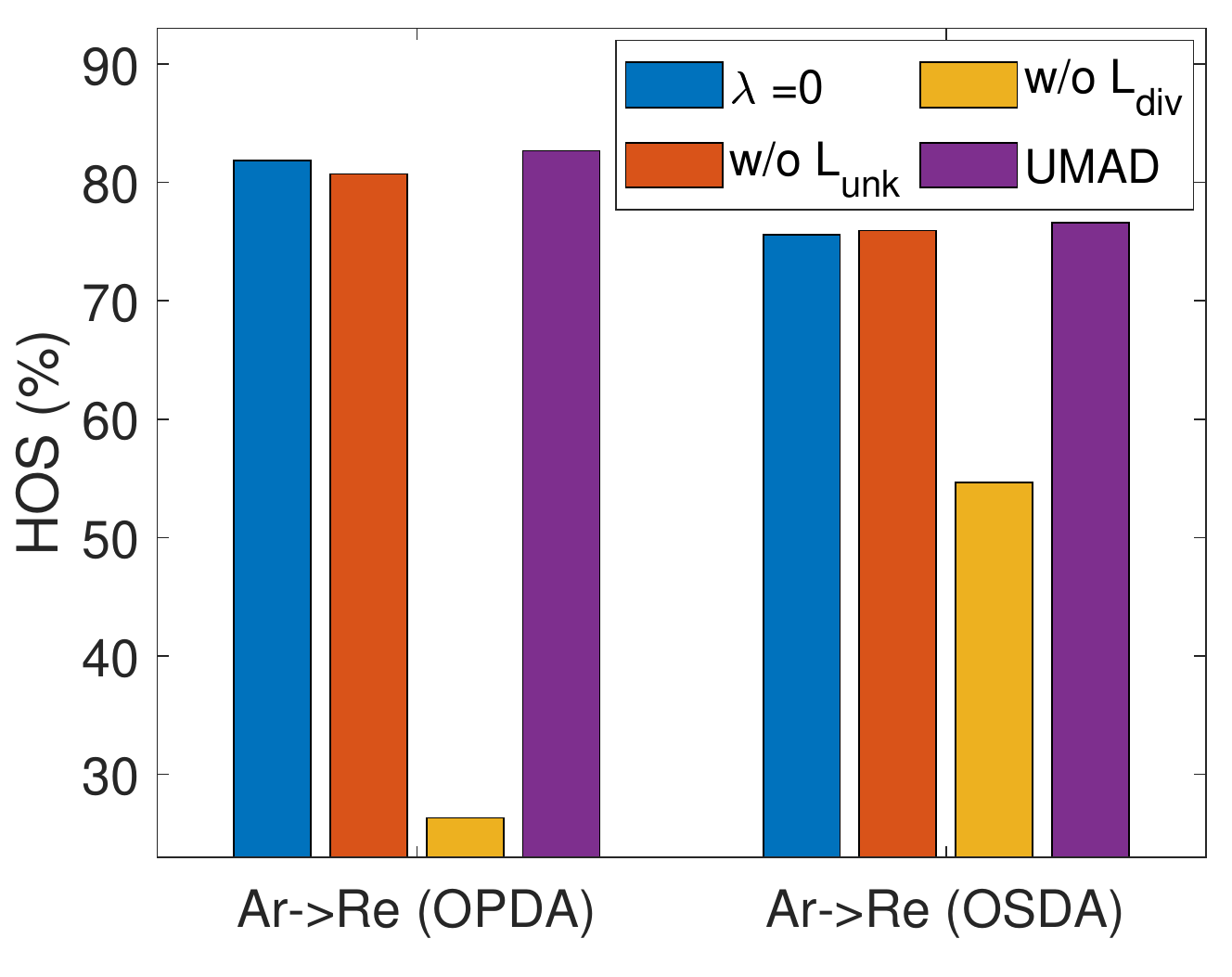} & 
			\includegraphics[width=0.245\linewidth,trim={3.3cm 9.2cm 4.2cm 9.8cm}, clip]{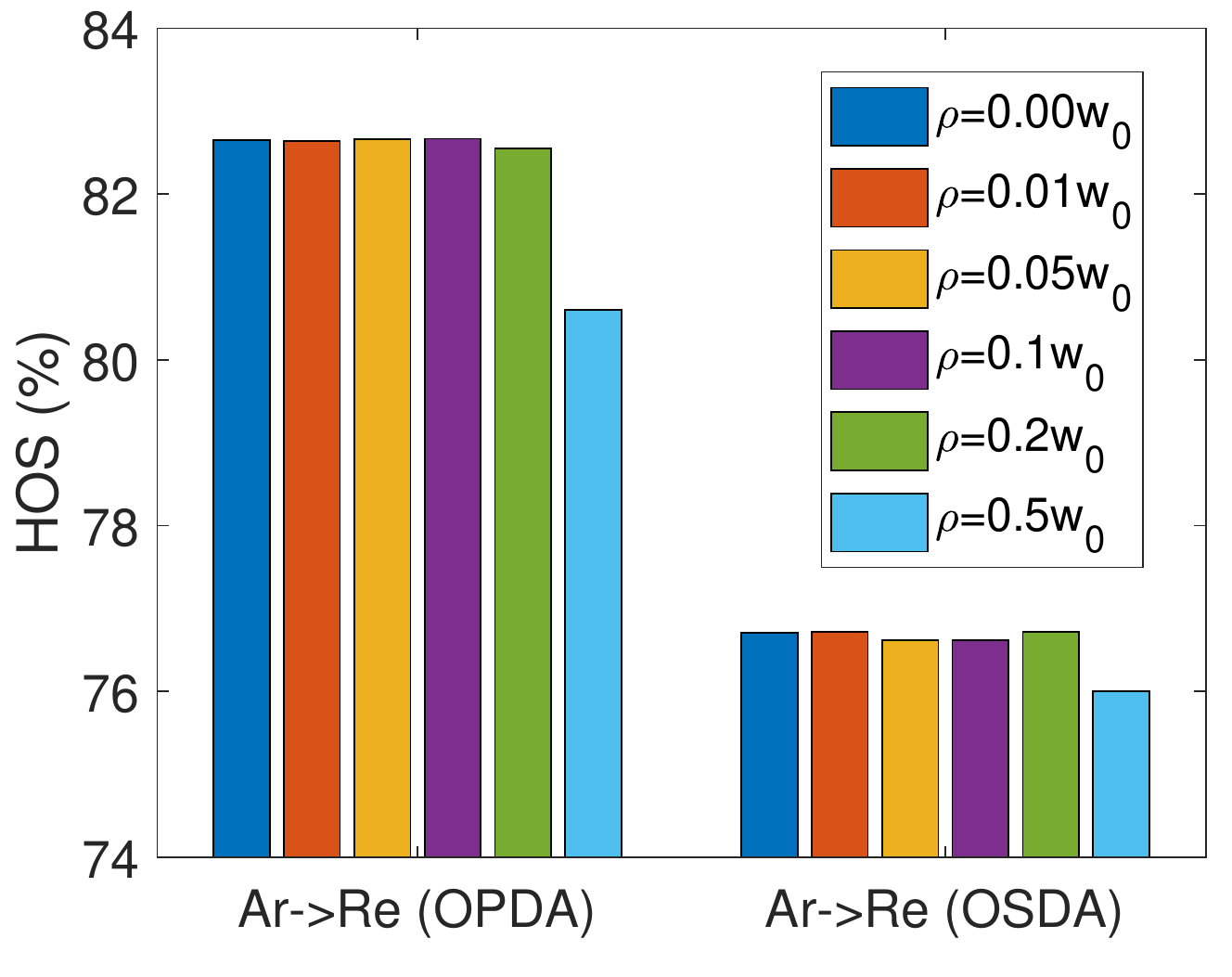} &
			\includegraphics[width=0.245\linewidth,trim={3.3cm 9.2cm 4.2cm 9.8cm}, clip]{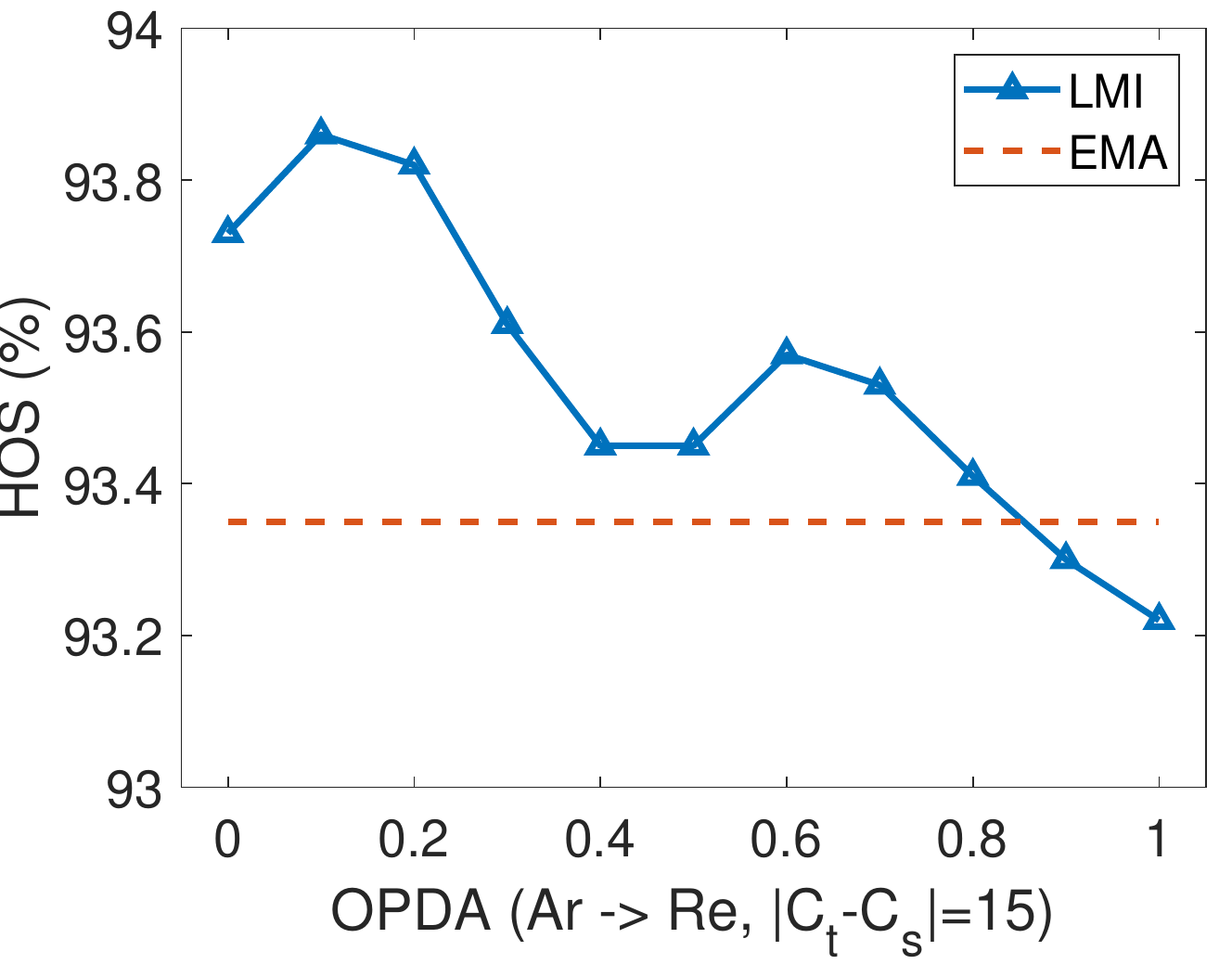} \\
			~\\
			(a) uncertainty score & (b) ablation on losses & (c) sensitivity to $\rho$ & (d) sensitivity to $T$
		\end{tabular}
		\caption{\textbf{Ablation study on different components of UMAD.} (a) shows results of other scores besides Eq.~(\ref{eq:score}), (b) shows the contribution of each component in the final objective, and (c-d) plot the sensitivity to two different parameters.}
		\label{fig:par_ab}
	\end{figure*} 
	
	\subsection{Analysis}
	\textbf{Varying unknown classes.} We compare UMAD with other methods with varying unknown classes $|\mathcal{C}_t - \mathcal{C}_{s}|$, and report results on \textbf{Office-Home} in Fig.~\ref{fig:par_num}. 
	The compared methods include universal UDA methods like DANCE and OVANet, OPDA method like CMU, and the source data-free OSDA method SHOT. 
	With an increasing number of unknown classes, UMAD stably outperforms most, if not all, other methods on both UDA scenarios, benefited from the proposed rejection mechanism. 
	
	\textbf{Choice of the uncertainty score.} We compare the performance of UMAD with different uncertainty scores in Fig.~\ref{fig:par_ab} (a), including the $L_2$ distance, the cosine distance, the mean entropy, and our proposed inner product distance. Results further demonstrate that only mean entropy and the inner product distance are effective and inner product consistently outperforms mean entropy on both UDA scenarios.
	
	\textbf{Ablation on losses.} We provide the ablation on main loss objectives in UMAD, including the two-classifier orthogonal constraint, the novel diversity term, and the entropic loss. Results in Fig.~\ref{fig:par_ab} (b) demonstrate the importance of each loss.  A significant decrease in the HOS score is expected with the removal of entropic loss. Because low accuracy of unknown samples will directly lead to low HOS score, according to Eq.~(\ref{eq:hos}).
	
	\textbf{Parameter sensitivity.} We study the parameter sensitivity of $\rho$ and $T$ in Fig.~\ref{fig:par_ab} (c-d). We study the ratio of the slack margin in a wide range, \ie, [0.00, 0.01, 0.05, 0.1, 0.2, 0.5], as shown in Fig.~\ref{fig:par_ab} (c). Results show that the performance of both UDA scenarios around the chosen $\rho$ is stable and may be better. In Fig.~\ref{fig:par_ab} (d), we study the flattening factor $T$ in the range of [0.0, 0.1, 0.2, 0.3, 0.4, 0.5, 0.6, 0.7, 0.8, 0.9, 1.0] and make comparisons with the vanilla EMA strategy~\cite{li2020rethinking}. The smaller $T$, the larger degree of flattening. Results show that within the value of $0.2$, the performance of UMAD is not sensitive to $T$.

	\setlength{\tabcolsep}{6.0pt}
	\begin{table}[!tb]
		\centering
		\caption{Avg. accuracy (\%) on \textbf{Office-Home} for partial-set UDA.} 
		\label{tab:office-distance}
		\begin{tabular}{lc}
			\toprule
			Method & Avg. \\
			\midrule
			MCC$^\dagger$~\cite{jin2020minimum} & 75.1 \\
			JUMBOT$^\dagger$~\cite{fatras2021unbalanced} & 75.5 \\
			BA$^3$US$^\dagger$~\cite{liang2020balanced} & 76.0 \\
			DCC$^\dagger$ \cite{li2021domain} & 73.0 \\
			DANCE~\cite{saito2021universal} & 71.1 \\
			\midrule
			No Adaptation & 64.4 \\
			UMAD ($w_0=0$, MI\cite{liang2020we}) & 72.9 \\
			UMAD ($w_0=0$, EMA\cite{li2020rethinking}) & \textbf{\color{bblue}78.0} \\
			UMAD ($w_0=0$, LMI) & \textbf{\color{blush}78.4} \\
			\bottomrule
		\end{tabular}
	\end{table}

	\textbf{Extension to partial-set UDA.} Besides the two UDA scenarios with unknown target samples investigated above, we ignore the rejection mechanism and apply UMAD to partial-set UDA \cite{liang2020balanced,jin2020minimum} on \textbf{Office-Home} and report the average accuracy across 12 tasks over three runs in Table~\ref{tab:office-distance}.
	We find our UMAD outperforms other data-dependent UDA methods including one of the state-of-the-art PDA methods---BA$^3$US \cite{liang2020balanced}. 
	Comparisons between different diversity terms show that our proposed LMI is the best and significantly outperforms the vanilla mutual information term~\cite{liang2020we} on partial-set UDA tasks.
	
	\textbf{Visualization of \textbf{iscore}.} 
	Fig.~\ref{fig:kde} (a) confirms that \textbf{iscore} in Eq.~(\ref{eq:score}) is meaningful and the devised ``synthesis" solution in Eq.~(\ref{eq:sys}) is reasonable.
	Fig.~\ref{fig:kde} (b) validates UMAD enlarges the discrepancy between known and unknown samples.
	
	\textbf{Limitations.}
	So far, it is hard to conduct closed-set UDA validation methods \cite{morerio2018minimal,you2019towards,saito2021tune} for UDA with unknown target samples. 
	Besides, UMAD works well for both open-set and open-partial-set scenarios, but it still needs adjustment when adapting to a closed-set or partial-set scenario.

	\begin{figure}[!tb]
		\centering
		\small
		\setlength\tabcolsep{1mm}
		\renewcommand\arraystretch{1.0}
		\begin{tabular}{c}
			\includegraphics[width=0.8\linewidth,trim={0.8cm 0.0cm 1.0cm 2.2cm}, clip]{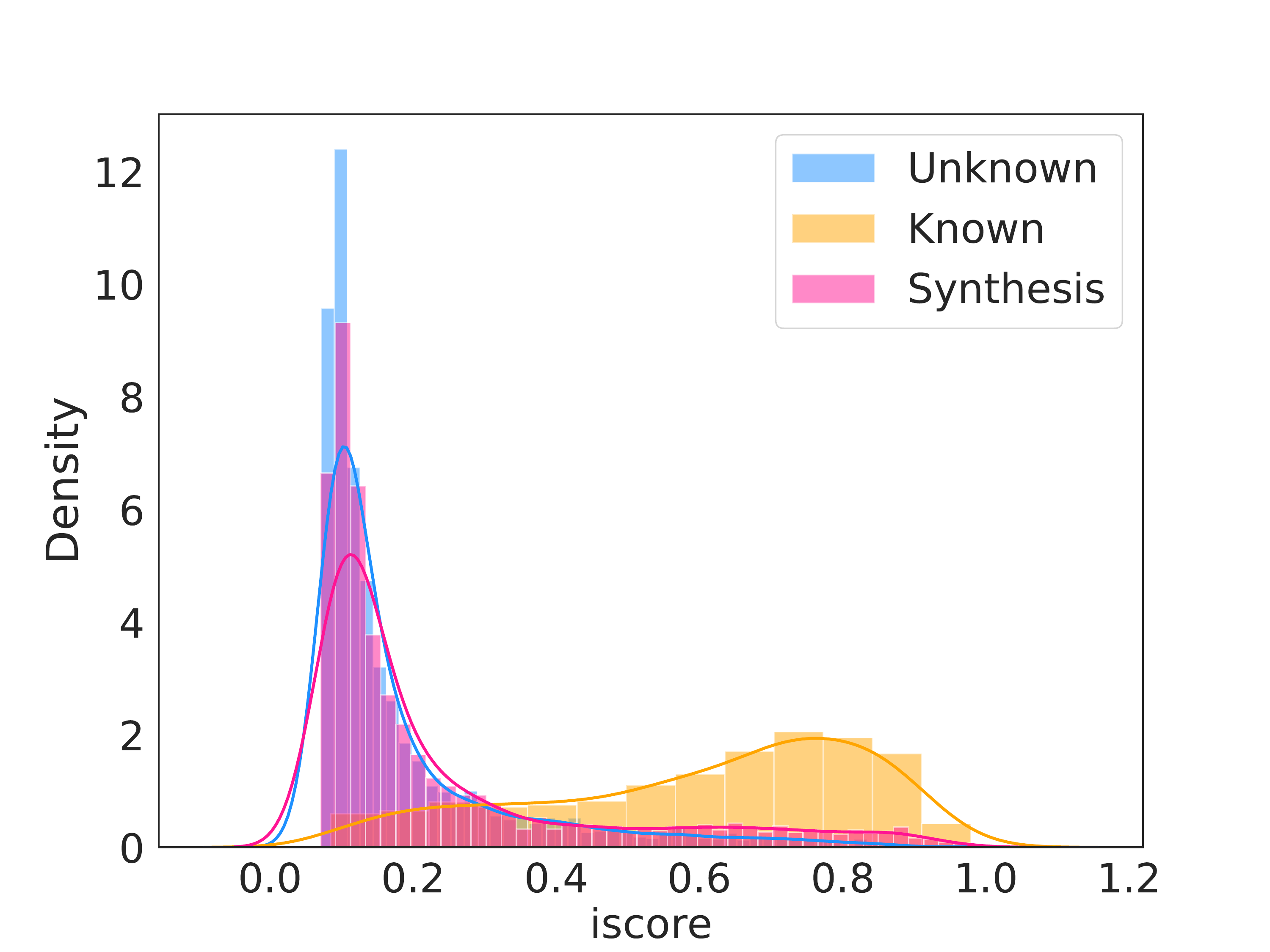} \\
			(a) No adaptation \\
			\includegraphics[width=0.8\linewidth,trim={0.8cm 0.0cm 1.0cm 2.2cm}, clip]{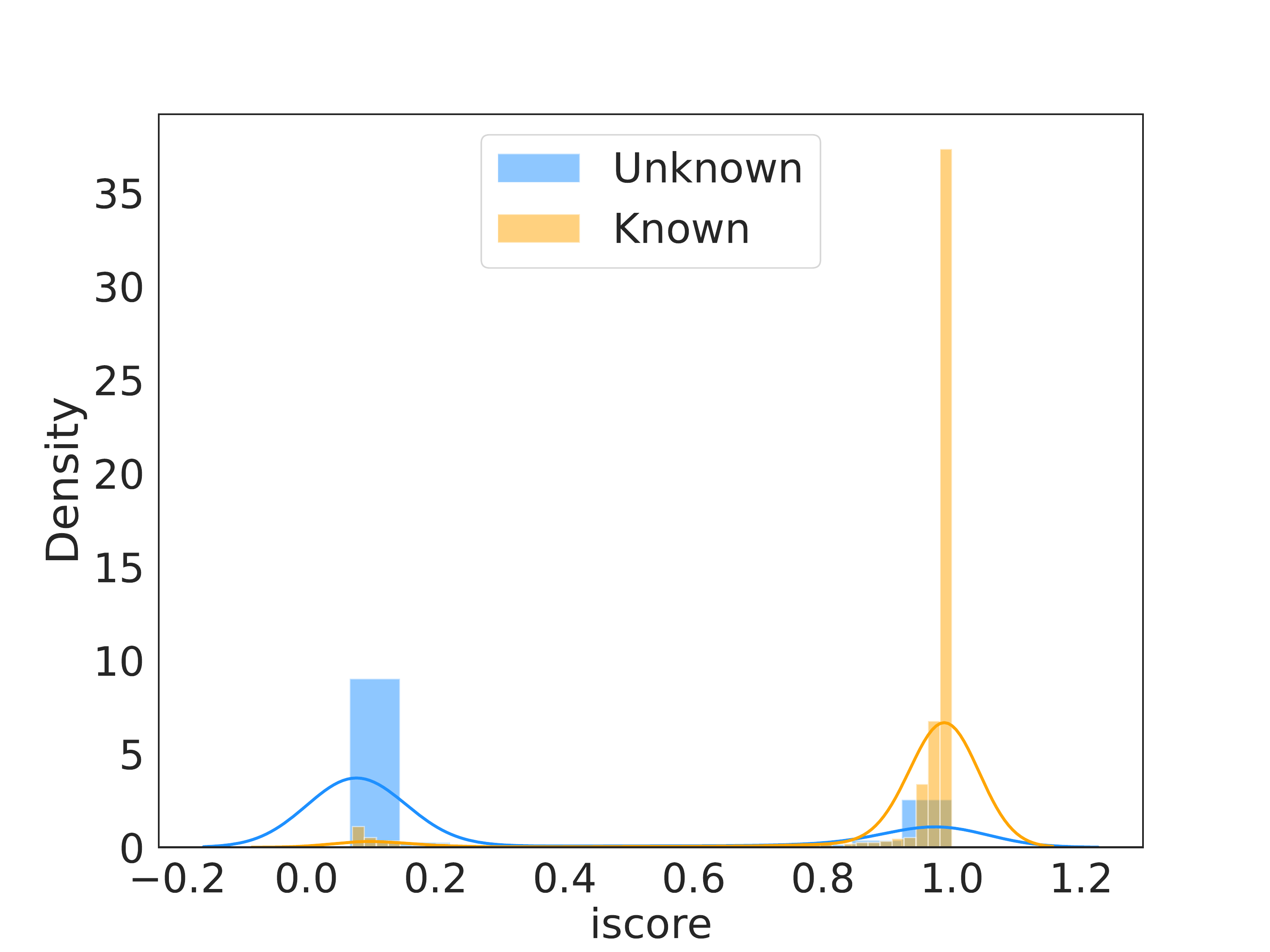} \\
			(b) UMAD \\
		\end{tabular}
		\caption{Histogram before/ after UMAD for Ar$\to$Re (OPDA).}
		\label{fig:kde}
	\end{figure} 

	\section{Conclusion}
	This paper investigated a more realistic open-set UDA setting where neither source data nor the prior about the label set overlap across domains is utilized during the target adaptation process.
	Using the trained model from the source domain, we proposed Universal Model ADaptation (UMAD).
	UMAD tackled category shift by devising an effective consistency score and an automatic thresholding scheme to detect samples from unknown classes.
	Besides, UMAD exploited a new localized mutual information objective to encourage target samples from known classes to fit the source model against domain shift.
	As a unified approach, UMAD has been validated to be effective across both open-set and open-partial-set UDA tasks on multiple datasets.
	As future work, UMAD can be extended to tackle complex tasks like segmentation and detection.

	\section{Notes}
	The early version of our manuscript (dubbed BATMAN) has been finished at Jan 2021. 
	Code to reproduce experimental results will be released at \url{https://github.com/tim-learn/UMAD}, questions are welcome.
	
	{\small
		\bibliographystyle{ieee_fullname}
		\bibliography{my}
	}
	
\end{document}